\documentclass[letterpaper]{article}
\usepackage[margin=1in,dvips]{geometry}
\usepackage{graphicx,psfrag,amsmath,amsthm,amssymb}
\usepackage{natbib}
\usepackage{url,color}

\newcommand{\F}{\ensuremath{\mathbf{F}}}
\newcommand{\G}{\ensuremath{\mathbf{G}}}

\newcommand{\I}{\ensuremath{\mathbf{I}}}

\newcommand{\U}{\ensuremath{\mathbf{U}}}
\newcommand{\V}{\ensuremath{\mathbf{V}}}
\newcommand{\W}{\ensuremath{\mathbf{W}}}
\newcommand{\X}{\ensuremath{\mathbf{X}}}
\newcommand{\Y}{\ensuremath{\mathbf{Y}}}
\newcommand{\Z}{\ensuremath{\mathbf{Z}}}
\renewcommand{\aa}{\ensuremath{\mathbf{a}}}

\renewcommand{\c}{\ensuremath{\mathbf{c}}}

\newcommand{\e}{\ensuremath{\mathbf{e}}}
\newcommand{\f}{\ensuremath{\mathbf{f}}}

\newcommand{\h}{\ensuremath{\mathbf{h}}}

\newcommand{\p}{\ensuremath{\mathbf{p}}}

\newcommand{\w}{\ensuremath{\mathbf{w}}}
\newcommand{\x}{\ensuremath{\mathbf{x}}}
\newcommand{\y}{\ensuremath{\mathbf{y}}}
\newcommand{\z}{\ensuremath{\mathbf{z}}}
\newcommand{\0}{\ensuremath{\mathbf{0}}}

\newcommand{\blambda}{\ensuremath{\boldsymbol{\lambda}}}

\newcommand{\btheta}{\ensuremath{\boldsymbol{\theta}}}

\newcommand{\bDelta}{\ensuremath{\boldsymbol{\Delta}}}

\newcommand{\bPi}{\ensuremath{\boldsymbol{\Pi}}}

\newcommand{\bTheta}{\ensuremath{\boldsymbol{\Theta}}}

\newcommand{\bbR}{\ensuremath{\mathbb{R}}}

\newcommand{\calB}{\ensuremath{\mathcal{B}}}
\newcommand{\calC}{\ensuremath{\mathcal{C}}}

\newcommand{\calF}{\ensuremath{\mathcal{F}}}

\newcommand{\calK}{\ensuremath{\mathcal{K}}}
\newcommand{\calL}{\ensuremath{\mathcal{L}}}

\newcommand{\calX}{\ensuremath{\mathcal{X}}}
\newcommand{\calY}{\ensuremath{\mathcal{Y}}}

\newcommand{\abs}[1]{\left\lvert#1\right\rvert}
\newcommand{\norm}[1]{\left\lVert#1\right\rVert}

\newcommand{\caja}[4][1]{{%
    \renewcommand{\arraystretch}{#1}%
    \begin{tabular}[#2]{@{}#3@{}}%
      #4%
    \end{tabular}%
    }}

{%
\begin{list}{#1}{
\vspace{-\topsep}
\vspace{-\partopsep}
\setlength{\itemindent}{0cm}
\setlength{\rightmargin}{0cm}
\setlength{\listparindent}{0cm}
\settowidth{\labelwidth}{#1}
\setlength{\leftmargin}{\labelwidth}
\addtolength{\leftmargin}{\labelsep}
\setlength{\itemsep}{0cm}
}%
}%
{%
\end{list}
\vspace{-\topsep}
\vspace{-\partopsep}
}

{\begin{enumerate}%
}%
{\end{enumerate}}

\DeclareMathOperator*{\argmin}{arg\,min}

\newcommand{\Eop}{\operatorname{E}}
\newcommand{\mean}[2][]{\ensuremath{\Eop_{#1}\left\{#2\right\}}}

\newcommand{\sgnop}{\operatorname{sgn}}
\newcommand{\sgn}[1]{\ensuremath{\sgnop\left(#1\right)}}

\theoremstyle{plain}
\newtheorem{thm}{Theorem}[section]
\newtheorem{lemma}[thm]{Lemma}
\newtheorem*{lemma*}{Lemma}

\newtheorem*{prop*}{Proposition}

\theoremstyle{definition}

\newtheorem*{defn*}{Definition}

\newtheorem*{exmp*}{Example}

\newtheorem*{conj*}{Conjecture}

\theoremstyle{remark}
\newtheorem{rmk}[thm]{Remark}
\newtheorem*{rmk*}{Remark}

\graphicspath{{grf-miguel/}}

\bibpunct[, ]{(}{)}{;}{a}{,}{,} 

\title{Model compression as constrained optimization, \\ with application to neural nets. \\ Part I: general framework.}
\author{
  Miguel \'A.\ Carreira-Perpi\~n\'an \\
  Electrical Engineering and Computer Science, University of California, Merced \\
  {\url{http://eecs.ucmerced.edu}}
}
\date{July 4, 2017}

\begin{document}

\maketitle

\begin{abstract}

  Compressing neural nets is an active research problem, given the large size of state-of-the-art nets for tasks such as object recognition, and the computational limits imposed by mobile devices. We give a general formulation of model compression as constrained optimization. This includes many types of compression: quantization, low-rank decomposition, pruning, lossless compression and others. Then, we give a general algorithm to optimize this nonconvex problem based on the augmented Lagrangian and alternating optimization. This results in a ``learning-compression'' algorithm, which alternates a learning step of the uncompressed model, independent of the compression type, with a compression step of the model parameters, independent of the learning task. This simple, efficient algorithm is guaranteed to find the best compressed model for the task in a local sense under standard assumptions.

  We present separately in several companion papers the development of this general framework into specific algorithms for model compression based on quantization, pruning and other variations, including experimental results on compressing neural nets and other models.

\end{abstract}

\section{Introduction}
\label{s:intro}

Large neural network models have become a central component in state-of-the-art practical implementations of the solution to various machine learning and artificial intelligence problems. These include, for example, classification problems involving images, audio or text, or reinforcement learning problems involving game playing or robot manipulation and navigation. This has also resulted in an enormous increase in the interest in deep neural net models from researchers in academy and industry, and even from non-experts and the public in general, as evidenced by the amount of scientific papers, blog entries or mainstream media articles published about it.

These practical successes in difficult problems have ocurred thanks to the availability of large-scale datasets and of massive computational power provided by GPUs, which are particularly well suited for the kind of linear algebra operations involved in training a neural net (such as stochastic gradient descent). One notable characteristic of deep neural nets that seems to distinguish them from other machine learning models is their ability to grow with the data. That is, as the size of the training set grows, we can continue to increase the (say) classification accuracy by increasing the size of the neural net (number of hidden units and of layers, and consequently the number of weights). This is unlike, for example, a linear model, whose classification accuracy will quickly stagnate as the data keeps growing. Hence, we can continue to improve the accuracy of a neural net by making it bigger and training on more data. This means that we can expect to see ever larger neural nets in future practical applications. Indeed, models reported in the literature of computer vision have gone from less than a million weights in the 1990s to millions in the 2000s and, in recent works, to models exceeding billions of weights (each a floating-point value).

The large size of the resulting model does cause an important practical problem when one intends to deploy it in a resource-constrained target device such as a mobile phone or other embedded systems. That is, the large neural net is trained in a resource-rich setting, e.g.\ GPUs and a multicore architecture with large memory and disk, where the model designer can explore different model architectures, hyperparameter values, etc.\ until a model with the best accuracy on a validation set is found. This final, large model is then ready to be deployed in practice. However, the target device will typically have far more restrictive computation constraints in memory size and speed, arithmetic operations, clock rate, energy consumption, etc., which make it impossible to accommodate the large model. In other words, we can only deploy models of a certain size. Download bandwidth is also significantly limited in apps for mobile phones or software for cars.

This problem has attracted considerable attention recently. Two important facts weigh upon it which have been known for some time among researchers. The first is that the large neural nets that we currently know how to train for optimal accuracy contain significant redundancy, which makes it possible to find smaller neural nets with comparable accuracy. The second is that, for reasons not entirely understood, one typically achieves a more accurate model by training a large model first and then somehow transforming it into a smaller one (``compressing'' the model), than by training a small model in the first place. This leads us to the problem of \emph{compressing a neural net} (or other model), which is our focus.

Compressing neural nets has been recognized in recent years as an important problem and various academic and industrial research groups have shown that one can indeed significantly compress neural nets without appreciable losses in accuracy (see related work below). However, the solutions proposed so far are somewhat ad-hoc in two senses. First, they define a specific compression technique (and a specific algorithm to find the compressed model) which may work well with some types of models but not others. Second, some of these solutions are not guaranteed to be optimal in the sense of achieving the highest classification accuracy for the compression technique considered.

In this paper, we provide a general formulation of model compression and a training algorithm to solve it. The formulation can accommodate any compression technique as long as it can be put in a certain mathematical form, which includes most existing techniques. The compression mechanism appears as a black box, so that the algorithm simply iterates between learning a large model and compressing it, but is guaranteed to converge to a locally optimal compressed model under some standard assumptions. In separate papers, we develop specific algorithms for various compression forms, such as quantization \citep{CarreirIdelbay17a} or pruning \citep{CarreirIdelbay17b}, and evaluate them experimentally. In the rest of this paper, we discuss related work (section~\ref{s:related}), give the generic formulation (section~\ref{s:formulation}) and the generic LC algorithm (section~\ref{s:LC}), give conditions for convergence (section~\ref{s:conv}) and discuss the relation of our LC algorithm with other algorithms (section~\ref{s:related-alg}) and the relation with generalization and model selection (section~\ref{s:gen}).

\section{Related work: what does it mean to compress a model?}
\label{s:related}

In a general sense, we can understand model compression as replacing a ``large'' model with a ``small'' model within the context of a given task, and this can be done in different ways. Let us discuss the setting of the problem that motivates the need for compression.

Assume we define the machine learning task as classification for object recognition from images and consider as large model $\f(\x;\w)\mathpunct{:}\ \calX \rightarrow \calY$ a deep neural net with inputs $\x \in \calX$ (images), outputs $\y \in \calY$ (object class labels) and real-valued weights $\w \in \bbR^P$ (where the number of weights $P$ is large). In order to train a model we use a loss $L()$, e.g.\ the cross-entropy on a large training set of input-output pairs $(\x_n,\y_n)$. Also assume we have trained a large, \emph{reference} neural net $\f(\x;\overline{\w})$, i.e., $\overline{\w} = \argmin_{\w}{ L(\f(\cdot;\w)) }$, and that we are happy with its performance in terms of accuracy, but its size is too large.

We now want a smaller model $\h(\x;\bTheta)\mathpunct{:}\ \calX \rightarrow \calY$ that we can apply to the same task (classifying input images \x\ into labels \y). How should we define this smaller model? One possibility is for the small model $\h(\x;\bTheta)$ to be of the same type as the reference model \f\ but reparameterized in terms of a low-dimensional parameter vector. That is, we construct weights $\w \in \bbR^P$ from low-dimensional parameters $\bTheta \in \bbR^Q$ via some transformation $\w = \bDelta(\bTheta)$ so that the size of \bTheta\ is smaller than the size of \w, i.e., $Q < P$ (obviously, this will constrain the possible $\w \in \bbR^P$ that can be constructed). In the example before, $\h(\x;\bTheta)$ would be the same deep net but, say, with weight values $\w = \bDelta(\bTheta)$ quantized using a codebook with $K$ entries (so \bTheta\ contains the codebook and the assignment of each weight $w_i$ to a codebook entry). A second possibility is for the small model $\h(\x;\bTheta)$ to be a completely different type of model from the reference, e.g.\ a linear mapping with parameters \bTheta\ (so there is no relation between \w\ and \bTheta).

Given this, we have the following options to construct the small model $\h(\x;\bTheta)$:
\begin{itemize}
\item \emph{Direct learning}: $\min_{\bTheta}{ L(\h(\x;\bTheta)) }$: \emph{find the small model with the best loss regardless of the reference}. That is, simply train the small model to minimize the loss directly, possibly using the chain rule to compute the gradient wrt \bTheta. This approximates the reference model indirectly, in that both \f\ and \h\ are trying to solve the same task. We can have the small model \h\ be a reparameterized version of \f\ or a completely different model. Direct learning is the best thing to do sometimes but not always, as noted later.
\item \emph{Direct compression (DC)}: $\smash{\min_{\bTheta}{ \norm{\overline{\w} - \bDelta(\bTheta)}^2 }}$: \emph{find the closest approximation to the parameters of the reference model}, using a low-dimensional parameterization $\bDelta(\bTheta)$. This forces \h\ and \f\ to be models of the same type. Direct compression can be simply done with (lossless or lossy) compression of $\overline{\w}$, but it generally will not be optimal wrt the loss since the latter is ignored in learning \bTheta. We discuss this later in detail.
\item \emph{Model compression as constrained optimization}: this is our proposed approach, which we describe in section~\ref{s:formulation}. It forces \h\ and \f\ to be models of the same type, by constraining the weights \w\ to be constructed from a low-dimensional parameterization $\w = \bDelta(\bTheta)$, but \h\ must optimize the loss $L$.
\item \emph{Teacher-student}: $\smash{\min_{\bTheta}{ \int_{\calX}{ p(\x) \, \norm{\f(\x;\overline{\w}) - \h(\x;\bTheta)}^2 \, d\x } }}$: \emph{find the closest approximation \h\ to the reference function \f}, in some norm. The norm over the domain \calX\ may be approximated with a sample (e.g.\ the training set). Here, the reference model \f\ ``teaches'' the student model \h. We can have the small model \h\ be a reparameterized version of \f\ or a completely different model. 
\end{itemize}
Most existing compression approaches fall in one of these categories. In particular, traditional compression techniques have been applied using techniques most related to direct training and direct compression: using low-precision weight representations through some form of rounding (see \citealp{Gupta_15a,Hubara_16b} and references therein), even single-bit (binary) values \citep{Fiesler_90a,Courbar_15a,Rasteg_16a,Hubara_16b,Zhou_16b}, ternary values \citep{HwangSung14a,Li_16b,Zhu_17a} or powers of two \citep{Marches_93a,TangKwan93a}; quantization of weight values, soft \citep{NowlanHinton92a,Ullric_17a} or hard \citep{Fiesler_90a,Marches_93a,TangKwan93a,Gong_15a,Han_15a}; zeroing weights to achieve a sparse model \citep{HansonPratt89a,Weigen_90a,Lecun_90a,HassibStork93a,Reed93a,Yu_12b,Han_15a}; low-rank factorization of weight matrices \citep{Sainat_13a,Denil_13a,Jaderb_14a,Denton_14a,Novikov_15a}; hashing \citep{Chen_15a}; and lossless compression, such as Huffman codes \citep{Han_16a}. Some papers combine several such techniques to produce impressive results, e.g.\ \citet{Han_16a} use pruning, trained quantization and Huffman coding. Although we comment on some of these works in this paper (particularly regarding the direct compression approach), we defer detailed discussions to our companion papers on specific compression techniques (quantization in \citealp{CarreirIdelbay17a} and pruning in \citealp{CarreirIdelbay17b}, so far).

The teacher-student approach seems to have arisen in the ensemble learning literature, inspired by the desire to replace a large ensemble model (e.g.\ a collection of decision trees), which requires large storage and is slow at test time, with a smaller or faster model with similar classification accuracy \citep[section~8.5]{Zhou12a}. The smaller model can be of the same type as the ensemble members, or a different type of model altogether (e.g.\ a neural net). The basic idea is, having trained the large ensemble on a labeled training set, to use this ensemble (the ``teacher'') to label a larger, unlabeled dataset (which may be available in the task at hand, or may be generated by some form of sampling). This larger, labeled dataset is then used to train the smaller model (the ``student''). The hope is that the knowledge captured by the teacher is adequately represented in the synthetically created training set (although the teacher's mistakes will also be represented in it), and that the student can learn it well even though it is a smaller model. More generally, the approach can be applied with any teacher, not necessarily an ensemble model (e.g.\ a deep net), and the teacher's labels may be transformed to improve the student's learning, e.g.\ log-outputs \citep{BaCaruan14a} or other manipulations of the output class probabilities \citep{Hinton_15a}. However, from a compression point of view the results are unimpressive, with modest compression ratios \citep{Hinton_15a} or even failure to compress (using a single-layer net as student needs many more weights than the teacher deep net; \citealp{BaCaruan14a}). One practical problem with the teacher-student approach is that all it really does is construct the artificial training set, but it leaves the design of the student to the user, and this is a difficult model selection problem. This is unlike the compression approaches cited above, which use the teacher's model architecture but compress its parameters.

\section{A constrained optimization formulation of model compression}
\label{s:formulation}

\begin{figure}[b!]
  \vspace*{4ex}
  \centering
  \begin{tabular}{@{}c@{\hspace{0.04\linewidth}}c@{\hspace{0.04\linewidth}}c@{}}
    \psfrag{XX}[][]{~~\caja{c}{c}{$\overline{\w}$ \\ (reference)}}
    \psfrag{FF}[l][Bl]{\caja{c}{c}{$\w^*$ (optimal \\ compressed)}}
    \psfrag{Fd}[l][Bl]{\caja{c}{c}{$\bDelta(\bTheta^{\text{DC}})$ \\ (direct \\ compression)}}
    \psfrag{R}[l][l]{\w-space}
    \psfrag{mappings}[][]{\caja{c}{c}{feasible models \calC \\ (decompressible \\ by \bDelta)}}
    \includegraphics[height=0.35\linewidth,bb=210 580 435 838,clip]{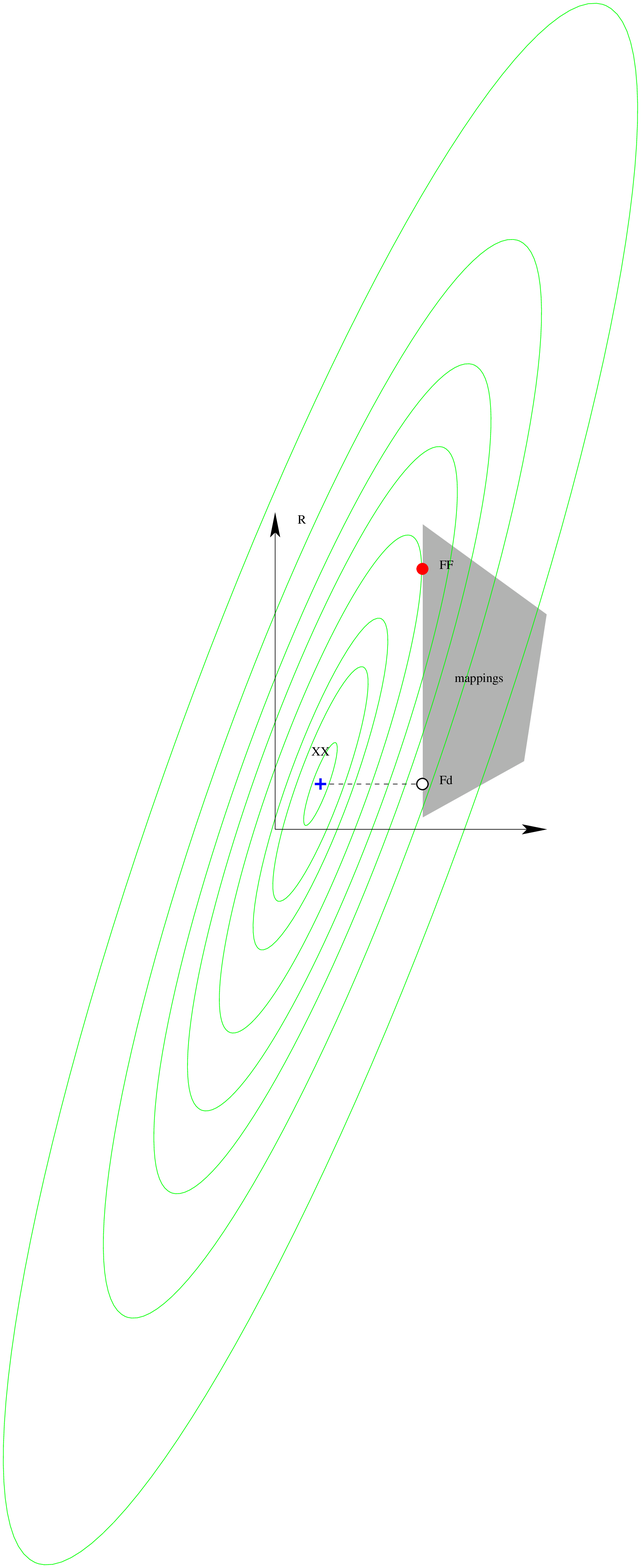} &
    \psfrag{XX1}[t][]{$\overline{\w}_1$}
    \psfrag{XX2}[t][]{$\overline{\w}_2$}
    \psfrag{FF}[l][Bl]{$\w^*$}
    \psfrag{Fd1}[l][Bl]{$\bDelta(\bTheta^{\text{DC}}_1)$}
    \psfrag{Fd2}[l][Bl]{$\bDelta(\bTheta^{\text{DC}}_2)$}
    \psfrag{R}[l][l]{\w-space}
    \psfrag{mappings}[][]{\calC}
    \includegraphics[height=0.35\linewidth,bb=15 27 240 285,clip]{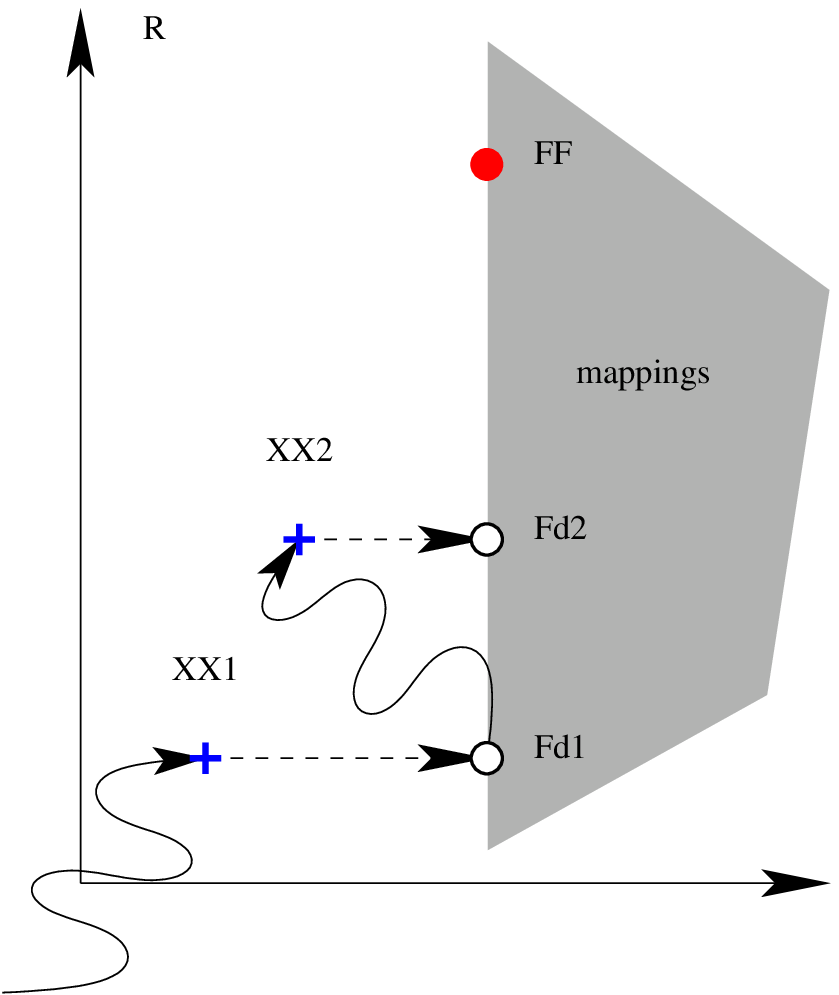} &
    \psfrag{XX}[t][]{$\overline{\w}$}
    \psfrag{FF1}[l][Bl]{$\w^*_1$}
    \psfrag{FF2}[Bl][l]{$\w^*_2$}
    \psfrag{FF3}[][l]{$\w^*_3$}
    \psfrag{FF4}[l][Bl]{$\w^*_4$}
    \psfrag{Fd1}[l][l][0.55]{$\bDelta(\bTheta^{\text{DC}}_1)$}
    \psfrag{Fd2}[l][l][0.55]{$\bDelta(\bTheta^{\text{DC}}_2)$}
    \psfrag{Fd3}[bl][l][0.55]{$\bDelta(\bTheta^{\text{DC}}_3)$}
    \psfrag{Fd4}[lt][l][0.55]{$\bDelta(\bTheta^{\text{DC}}_4)$}
    \psfrag{F1}[l][Bl]{$\calC_1$}
    \psfrag{F2}[l][Bl]{$\calC_2$}
    \psfrag{F3}[l][Bl]{$\calC_3$}
    \psfrag{F4}[l][Bl]{$\calC_4$}
    \psfrag{R}[l][l]{\w-space}
    \includegraphics[height=0.35\linewidth,bb=210 580 435 838,clip]{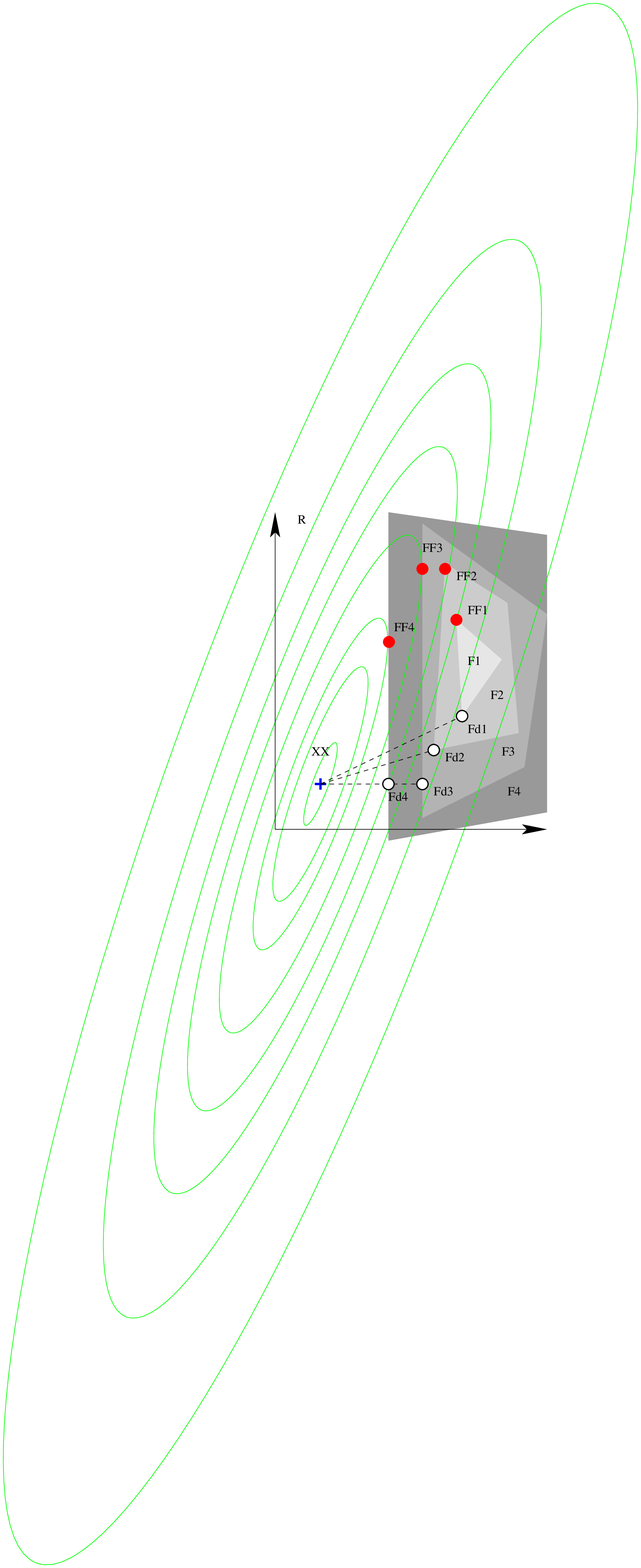}
  \end{tabular}
  \caption{Schematic representation of the idea of model compression by constrained optimization. The plots illustrate the uncompressed model space (\w-space $= \bbR^P$), the contour lines of the loss $L(\w)$ (green lines), and the set of compressed models (the feasible set $\calC = \{\w \in \bbR^P\mathpunct{:}\ \w = \bDelta(\bTheta) \text{ for } \bTheta \in \bbR^Q\}$, grayed areas), for a generic compression technique \bDelta. The \bTheta-space is not shown. $\overline{\w}$ optimizes $L(\w)$ but is infeasible (no \bTheta\ can decompress into it). The direct compression $\w^{\text{DC}} = \bDelta(\bTheta^{\text{DC}})$ is feasible but not optimal compressed (not optimal in the feasible set). $\w^* = \bDelta(\bTheta^*)$ is optimal compressed. Plot 2 shows two local optima $\overline{\w}_1$ and $\overline{\w}_2$ of the loss $L(\w)$, and their respective DC points (the contour lines are omitted to avoid clutter). Plot 3 shows several feasible sets, corresponding to different compression levels ($\calC_1$ is most compression).}
  \label{f:LC-illustration}
\end{figure}

In this work we understand compression in a mathematically specific sense that involves a learning task that we want to solve and a large model that sets the reference to meet. It is related to the direct training and direct compression concepts introduced above. Assume we have a large, \emph{reference} model $\f(\x;\overline{\w})$ with $P$ parameters (e.g.\ a neural net with inputs \x\ and weights $\w \in \bbR^P$) that has been trained on a \emph{loss} $L()$ (e.g.\ cross-entropy on a given training set) to solve a \emph{task} (e.g.\ classification). That is, $\overline{\w} = \argmin_{\w}{ L(\w) }$, where we abuse the notation to write the loss directly over the weights rather than $L(\f(\x;\w))$, and the minimizer $\overline{\w}$ may be local. We define \emph{compression} as finding a low-dimensional parameterization $\bDelta(\bTheta)$ of \w\ in terms of $Q < P$ parameters \bTheta. This will define a \emph{compressed} model $\h(\x;\bTheta) = \f(\x;\bDelta(\bTheta))$. \emph{We seek a \bTheta\ such that its corresponding model has (locally) optimal loss. We denote this``optimal compressed'' and write it as $\bTheta^*$ and $\w^* \equiv \bDelta(\bTheta^*)$ } (see fig.~\ref{f:LC-illustration}).

Ordinarily, one could then solve the problem directly over \bTheta: $\bTheta^* = \argmin_{\bTheta}{ L(\bDelta(\bTheta)) }$. This is the \emph{direct learning} option in the previous section. Instead, we equivalently write \emph{model compression as a constrained optimization} problem:
\begin{equation}
  \label{e:compression-problem}
  \textcolor{blue}{\min_{\w,\bTheta}{ L(\w) } \quad \text{s.t.} \quad \w = \bDelta(\bTheta).}
\end{equation}
The reason, which will be useful later to develop an optimization algorithm, is that we decouple the part of \emph{learning} the task, $L(\w)$ in the objective, from the part of \emph{compressing} the model, $\w = \bDelta(\bTheta)$ in the constraints.

By eliminating \w, our formulation~\eqref{e:compression-problem} is equivalent to direct learning $\min_{\bTheta}{ L(\bDelta(\bTheta)) }$, so why not do that in the first place, rather than training a large model and then compressing it? In fact, direct learning using gradient-based methods (via the chain rule) may sometimes be a good option. But it is not always convenient or possible. Firstly, if the decompression mapping \bDelta\ is not differentiable wrt \bTheta\ (as happens with quantization), then the chain rule does not apply. Second, using gradient-based methods over \bTheta\ may lead to slow convergence or be prone to local optima compared to training a large, uncompressed model (this has been empirically reported with pruning \citealp{Reed93a} and low-rank compression, e.g.\ \citealp{Denil_13a}). Third, the direct learning does not benefit from the availability of a large, well-trained model in \w-space, since it operates exclusively in the low-dimensional \bTheta-space. Finally, in direct learning the learning task aspects (loss $L$ and training set) are intimately linked to the compression ones (\bDelta\ and \bTheta), so that the design of a direct learning algorithm is specific to the combination of loss and compression technique (in our LC algorithm, both aspects separate and can be solved independently).

\subsection{Compression as orthogonal projection on the feasible set}
\label{s:compression-as-projection}

Compression and decompression are usually seen as algorithms, but here we regard them as mathematical mappings in parameter space. If
\begin{equation}
  \label{e:compression-mapping}
  \textcolor{blue}{\bPi(\w) = \argmin_{\bTheta}{\norm{\w - \bDelta(\bTheta)}^2}}
\end{equation}
is well defined, we call $\bPi\mathpunct{:}\ \w \in \bbR^P \rightarrow \bTheta \in \bbR^Q$ the \emph{compression mapping}. \bPi\ behaves as the ``inverse'' of the \emph{decompression mapping} $\bDelta\mathpunct{:}\ \bTheta \in \bbR^Q \rightarrow \w \in \bbR^P$ (although it is not a true inverse, because $\bDelta \circ \bPi \neq$ identity). Since $P > Q$, there generally will be a unique inverse \bTheta\ for any given \w\ (but not necessarily). Compressing \w\ may need an algorithm (e.g.\ SVD for low-rank compression, $k$-means for quantization) or a simple formula (e.g.\ taking the sign or rounding a real value). \bPi\ will usually satisfy $\bPi(\bDelta(\bTheta)) = \bTheta$ for any $\bTheta \in \bbR^Q$, i.e., decompressing \bTheta\ then compressing it gives back \bTheta. The decompression mapping \bDelta\ appears explicitly in our problem definition~\eqref{e:compression-problem}, while the compression mapping \bPi\ appears in our LC algorithm (in the C step), as we will see.

With lossless compression, \bDelta\ is bijective and $\bPi = \bDelta^{-1}$. With lossy compression, \bDelta\ need be neither surjective (since $Q < P$) nor injective (since it can have symmetries, e.g.\ reordering singular values in the SVD or centroids in $k$-means). Also, \bDelta\ need not be differentiable wrt \bTheta\ (for low-rank compression it is, for quantization it is not).

The feasible set $\calC = \{\w \in \bbR^P\mathpunct{:}\ \w = \bDelta(\bTheta) \text{ for } \bTheta \in \bbR^Q\}$ contains all high-dimensional models \w\ that can be obtained by decompressing some low-dimensional model \bTheta. In our framework, compression is equivalent to orthogonal projection on the feasible set. Indeed, $\bPi(\w) = \smash{\argmin_{\bTheta}{ \norm{\w - \bDelta(\bTheta)}^2 }}$ is equivalent to $\bPi(\w) = \smash{\argmin_{\bTheta,\w'}{ \norm{\w - \w'}^2 }}$ s.t.\ $\w' = \bDelta(\bTheta)$, which is the problem of finding the closest feasible point to \w\ in Euclidean distance, i.e., $\bPi(\w)$ is the orthogonal projection of \w\ on the feasible set.

\subsection{Types of compression}
\label{s:compression-types}

Our framework includes many well-known types of compression:
\begin{description}
\item[Low-rank compression] defines $\bDelta(\U,\V) = \smash{\U\V^T}$, where we write the weights in matrix form with \W\ of $m \times n$, \U\ of $m \times r$ and \V\ of $n \times r$, and with $r < \min(m,n)$. If we learn both \U\ and \V, the compression mapping is given by the singular value decomposition (SVD) of \W. We can also use a fixed dictionary or basis (e.g.\ given by wavelets or the discrete cosine transform) and learn either \U\ or \V\ only. The compression mapping is then given by solving a linear system. We study this in a paper in preparation.
\item[Quantization] uses a discrete mapping \bDelta\ given by assigning each weight to one of $K$ codebook values. If we learn both the assignments and the codebook, compression can be done by $k$-means. We can also use a fixed codebook, such as $\{-1,+1\}$ or $\{-1,0,+1\}$. The compression mapping is then given by a form of rounding. We study this in a separate paper \citep{CarreirIdelbay17a}.
\item[Low-precision approximation] defines a constraint $w_i = \theta_i$ per weight where $w_i$ is real (or, say, a double-precision float) and $\theta_i$ is, say, a single-precision float. The compression mapping sets $\theta_i$ to the truncation of $w_i$. A particular case is binarization, where $\theta_i \in \{-1,+1\}$ and the compression mapping sets $\theta_i = \sgn{w_i}$. This can be seen as quantization using a fixed codebook.
\item[Pruning] defines $\w = \bDelta(\btheta) = \btheta$ where \w\ is real and \btheta\ is constrained to have few nonzero values. The compression mapping involves some kind of thresholding. We study this in a separate paper \citep{CarreirIdelbay17b}.
\item[Lossless compression] takes many forms, such as Huffman coding, arithmetic coding or run-length encoding \citep{GershoGray92a}, and is special in that \bDelta\ is a bijection. Hence, the direct compression solves the problem with no need for our LC algorithm. However, lossless compression affords limited compression power.
\end{description}
It is also possible to combine several compression techniques.

For any lossy compression technique, the user can choose a \emph{compression level} (e.g.\ the rank in low-rank approaches or the codebook size in quantization approaches). Obviously, we are interested in the highest compression level that retains acceptable accuracy in the task. Note that in accounting for the size of the compressed model, we need to add two costs: storing the weights of the compressed model, and storing the decompression algorithm data (the dictionary, the codebook, the location of the nonzero values, etc.).

We can be flexible in the definition of the constraints in problem~\eqref{e:compression-problem}. We need not compress all the weights (e.g.\ we usually do not compress the biases in a neural net); this simply means we have no constraint for each such weight $w_i$. We can use different types of compression for different parts of the model (e.g.\ for different layers of the net); this means we have sets of constraints $\w_j = \bDelta_j(\bTheta_j)$, which separate and can be ran in parallel in the C step (see later). Also, there may be additional constraints in problem~\eqref{e:compression-problem}. For example, in quantization we have binary assignment vectors whose sum must equal one, or variables that must belong to a set such as $\{-1,+1\}$. The original minimization over the loss $L(\w)$ may also be constrained, e.g.\ if the weights should be nonnegative.

The decompression mapping $\w = \bDelta(\bTheta)$ can take different forms. Each $w_i$ may be a function of the entire \bTheta, which are then ``shared'' compression parameters. This happens with low-rank compression: $\W = \bDelta(\U,\V) = \U \V^T$. Instead, each $w_i$ may have ``shared'' parameters and ``private'' parameters $\vartheta_i$. This happens in quantization: $w_i = c_{\vartheta_i}$, where $\{c_1,\dots,c_K\}$ is a codebook shared by all weights and $\vartheta_i \in \{1,\dots,K\}$ is the index in the codebook that $w_i$ is assigned to%
\footnote{Actually, it is more convenient to express $\vartheta_i$ as a binary assignment vector $\z_i \in \{0,1\}^K$ where $\sum^K_{k=1}{ z_{ik} } = 1$, see \citet{CarreirIdelbay17a}.}.
Here, $\bTheta = \{c_1,\dots,c_K\} \cup \{\vartheta_i\}^P_{i=1}$.

Earlier we defined the feasible set $\calC = \{\w \in \bbR^P\mathpunct{:}\ \w = \bDelta(\bTheta) \text{ for } \bTheta \in \bbR^Q\}$, which contains all high-dimensional models \w\ that can be obtained by decompressing some low-dimensional model \bTheta. Good compression schemes should satisfy two desiderata:
\begin{enumerate}
\item Achieve low compression error. Obviously, this depends on the compression level (which determines the ``size'' of the feasible set) and on the optimization, e.g.\ our LC algorithm (which adapts the parameters \bTheta\ to the task at hand as best as possible). But the form of the compression mapping (low-rank, quantization, etc.\@) matters. This form should be such that every uncompressed model \w\ of interest is near some part of the feasible set, i.e., the decompression mapping $\bDelta(\bTheta)$ is ``space-filling'' to some extent.
\item Have simple compression and decompression algorithms: fast and ideally without local optima.
\end{enumerate}

\subsection{Other formulations of model compression}
\label{s:other}

\paragraph{A penalty formulation}

One can define the compression problem as
\begin{equation}
  \label{e:compression-penalty}
  \min_{\w}{ L(\w) + \lambda \, C(\w) }
\end{equation}
where the penalty or cost function $C(\w)$ encourages \w\ to be close to a compressed model and $\lambda \ge 0$ is a user parameter. Computationally, this can also be conveniently optimized by duplicating \w:
\begin{equation}
  \label{e:compression-penalty2}
  \min_{\w,\bTheta}{ L(\w) + \lambda \, C(\bTheta) } \quad \text{s.t.} \quad \w = \bTheta
\end{equation}
and applying a penalty method and alternating optimization, so the learning part on $L$ and \w\ separates from the compression part on \bTheta. However, this formulation is generally less preferable than the constrained formulation~\eqref{e:compression-problem}. The reason is that the penalty $\lambda \, C(\w)$ does not guarantee that the optimum of~\eqref{e:compression-penalty} is exactly a compressed model, only that it is close to some compressed model, and a final, suboptimal compression step is required. For example, penalizing the deviations of the weights $w_i$ from a given codebook (say, $\{-1,+1\}$) will encourage the weights to cluster around codebook values, but not actually to equal them, so upon termination we must round all weights, which makes the result suboptimal.

For some types of compression a penalty formulation does produce exactly compressed models. In pruning \citep{CarreirIdelbay17b}, we want the weight vector \w\ to contain many zeros (be sparse). Using a penalty $\lambda \, C(\w)$ where $C$ is a sparsity-inducing norm, say $C(\w) = \norm{\w}_0$, will result in a sparse weight vector. Still, in the penalty form the number of nonzeros in \w\ is implicitly related to the value of $\lambda$, while the constraint form $\norm{\w}_0 \le \kappa$ allows us to set the number of nonzeros directly, which is more convenient in practice.

\paragraph{Another constrained formulation}

It is conceivable to consider the following, alternative formulation of model compression as a constrained optimization:
\begin{equation}
  \label{e:compression-problem2}
  \min_{\w,\bTheta}{ L(\w) } \quad \text{s.t.} \quad \bPi(\w) = \bTheta
\end{equation}
directly in terms of a well-defined compression mapping \bPi, rather than in terms of a decompression mapping \bDelta\ as in~\eqref{e:compression-problem}. This has the advantage that problem~\eqref{e:compression-problem2} is simply solved by setting $\w^* = \overline{\w} \equiv \argmin_{\w}{ L(\w) }$ (the reference model) and $\bTheta^* = \bPi(\overline{\w})$, without the need for an iterative algorithm over \w\ and \bTheta\ (we call this ``direct compression'' later). Indeed, the constraint in~\eqref{e:compression-problem2} does not actually constrain \w. However, this formulation is rarely useful, because the resulting compressed model may have an arbitrarily large loss $L(\bPi(\overline{\w}))$. An exception is lossless compression, which satisfies $\bPi = \bDelta^{-1}$, and here the optimal compressed solution can indeed be achieved by compressing the reference model directly.

\section{A ``Learning-Compression'' (LC) algorithm}
\label{s:LC}

Although the constrained problem~\eqref{e:compression-problem} can be solved with a number of nonconvex optimization algorithms, it is key to profit from parameter separability, which we achieve with penalty methods and alternating optimization, as described next.

\paragraph{Handling the constraints via penalty methods}

Two classical penalty methods are the quadratic penalty (QP) and the augmented Lagrangian (AL) \citep{NocedalWright06a}. In the QP, we optimize the following over the parameters $(\w,\bTheta)$ while driving $\mu\rightarrow\infty$:
\begin{equation}
  \label{e:QP}
  Q(\w,\bTheta;\mu) = L(\w) + \frac{\mu}{2} \norm{\w - \bDelta(\bTheta)}^2.
\end{equation}
This has the effect of gradually enforcing the constraints, and the parameters trace a path $(\w(\mu),\bTheta(\mu))$ for $\mu > 0$. A better method is the AL. Here we optimize the following over $(\w,\bTheta)$ while driving $\mu\rightarrow\infty$:
\begin{align}
  \label{e:augLag}
  \calL_A(\w,\bTheta,\blambda;\mu) &= L(\w) - \blambda^T (\w - \bDelta(\bTheta)) + \frac{\mu}{2} \norm{\w - \bDelta(\bTheta)}^2 \\
  \label{e:augLag2}
  &= L(\w) + \frac{\mu}{2} \norm{\w - \bDelta(\bTheta) - \frac{1}{\mu} \blambda}^2 - \frac{1}{2\mu} \norm{\blambda}^2
\end{align}
and we update the Lagrange multiplier estimates $\blambda \leftarrow \blambda - \mu (\w - \bDelta(\bTheta))$ after optimizing $\calL_A$ over $(\w,\bTheta)$ for each $\mu$. Optimizing $\calL_A$ over $(\w,\bTheta)$ for fixed \blambda\ is like optimizing the QP with a shifted parameterization $\bDelta(\bTheta) + \frac{1}{\mu} \blambda$, as eq.~\eqref{e:augLag2} shows explicitly. The AL is equivalent to the QP if $\blambda = \0$.

\paragraph{Optimizing the penalized function with alternating optimization}

Applying alternating optimization to the penalized function (quadratic-penalty function $Q$ or augmented Lagrangian $\calL_A$) over \w\ and \bTheta\ (for fixed \blambda) gives our ``learning-compression'' (LC) algorithm. The steps are as follows:
\begin{itemize}
\item \textbf{L (learning) step:}
  \begin{equation}
    \label{e:Lstep}
    \textcolor{blue}{\min_{\w}{ L(\w) + \frac{\mu}{2} \norm{\w - \bDelta(\bTheta) - \frac{1}{\mu} \blambda}^2 }.}
  \end{equation}
  This is a regular training of the uncompressed model but with a quadratic regularization term. \emph{This step is independent of the compression type.}
\item \textbf{C (compression) step:}
  \begin{equation}
    \label{e:Cstep}
    \textcolor{blue}{\min_{\bTheta}{ \norm{\w - \frac{1}{\mu} \blambda - \bDelta(\bTheta)}^2 } \quad \Longleftrightarrow \quad \bTheta = \bPi \left(\w - \frac{1}{\mu} \blambda \right)}.
  \end{equation}
  This means finding the best (lossy) compression of $\w - \smash{\frac{1}{\mu}} \blambda$ (the current uncompressed model, shifted by $\smash{\frac{1}{\mu}} \blambda$) in the $\ell_2$ sense (orthogonal projection on the feasible set), and corresponds to our definition of the compression mapping \bPi\ in section~\ref{s:compression-as-projection}. \emph{This step is independent of the loss, training set and task.}
\end{itemize}

\subsection{Practicalities}

Fig.~\ref{f:pseudocode} gives the LC algorithm pseudocode for the augmented Lagrangian (using a single LC iteration per \blambda\ update). For the quadratic-penalty version, ignore or set to zero \blambda\ everywhere.

\paragraph{Reusing existing code in the L and C steps}

The L step gradually pulls \w\ towards a model that can be obtained by decompressing some $\bTheta \in \bbR^Q$, and the C step compresses the current \w. Both of these steps can be done by reusing existing code rather than writing a new algorithm, which makes the LC algorithm easy to implement. The L step just needs an additive term ``$\mu (\w - \bDelta(\bTheta) - \smash{\frac{1}{\mu}} \blambda)$'' to the gradient of $L(\w)$, e.g.\ in stochastic gradient descent (SGD) for neural nets. The C step depends on the compression type, but will generally correspond to a well-known compression algorithm (e.g.\ SVD for low-rank compression, $k$-means for quantization). Different types of compression can be used by simply calling a different compression routine in the C step, with no other change to the LC algorithm. This facilitates the model designer's job of trying different types of compression to find the most suitable for the task at hand.

\begin{figure}[b]
  \centering
  \setlength{\fboxsep}{1ex}
  \framebox{%
    \begin{minipage}[c]{0.80\linewidth}
      \begin{tabbing}
        n \= n \= n \= n \= n \= \kill
        \underline{\textbf{input}} training data and model with parameters (weights) \w \\
        $\w \leftarrow \overline{\w} = \argmin_{\w}{ L(\w) }$ \` {\small\textsf{reference model}} \\
        $\bTheta \leftarrow \bTheta^{\text{DC}} = \bPi(\overline{\w}) = \argmin_{\bTheta}{ \norm{\overline{\w} - \bDelta(\bTheta)}^2 }$ \` {\small\textsf{compress reference model}} \\
        $\blambda \leftarrow \0$ \\
        \underline{\textbf{for}} $\mu = \mu_0 < \mu_1 < \dots < \infty$ \+ \\
        $\w \leftarrow \argmin_{\w}{ L(\w) + \smash{\frac{\mu}{2} \norm{\smash{\w - \bDelta(\bTheta) - \frac{1}{\mu} \blambda}}^2} }$ \` {\small\textsf{L step: learn model}} \\
        $\bTheta \leftarrow \bPi(\w - \smash{\frac{1}{\mu} \blambda}) = \argmin_{\bTheta}{ \norm{\w - \smash{\frac{1}{\mu} \blambda} - \bDelta(\bTheta)}^2 }$ \` {\small\textsf{C step: compress model}} \\
        $\blambda \leftarrow \blambda - \mu (\w - \bDelta(\bTheta))$ \` {\small\textsf{Lagrange multipliers}} \\
        \textbf{if} $\norm{\w - \bDelta(\bTheta)}$ is small enough \textbf{then} exit the loop \- \\
        \underline{\textbf{return}} \w, \bTheta
      \end{tabbing}
    \end{minipage}
  }
  \caption{Pseudocode for the LC algorithm, augmented-Lagrangian version.}
  \label{f:pseudocode}
\end{figure}

\paragraph{Runtime}

The runtime of the C step is typically negligible compared to that of the L step (which involves the actual training set, usually much larger than the number of parameters), although this depends on the type of compression and the loss. Hence, it pays to do the C step as exactly as possible. The overall runtime of the LC algorithm will be dominated by the L steps, as if we were training an uncompressed model for a longer time.

\paragraph{Schedule of the penalty parameter $\mu_k$}

In practice, as usual with penalty methods, we use a multiplicative $\mu$ schedule: $\mu_k = a^k \mu_0$ with $a$ slightly larger than 1 and $\mu_0 \approx 0$ (set by trial and error); see also section~\ref{s:conv}. As noted in the pseudocode, we run a single L and C step per $\mu$ value because this keeps the algorithm simple (we avoid an extra loop). However, in some cases it may be advantageous to run multiple L and C steps per $\mu$ value, e.g.\ if it is possible to cache matrix factorizations in order to speed up the L or C step.

\paragraph{Initialization and termination}

We always initialize $\blambda = \0$ and $(\w,\bTheta) = (\overline{\w},\bTheta^{\text{DC}})$, i.e., to the reference model and direct compression, which is the exact solution for $\mu \rightarrow 0^+$, as we show in the next section. We stop the LC algorithm when $\norm{\w - \bDelta(\bTheta)}$ is smaller than a set tolerance, which will happen when $\mu$ is large enough. At this point, the final iterate $(\w,\bTheta)$ satisfies $\bTheta = \bPi(\w - \smash{\frac{1}{\mu} \blambda}) = \argmin_{\bTheta}{ \norm{\w - \smash{\frac{1}{\mu} \blambda} - \bDelta(\bTheta)}^2 }$, so that $\bDelta(\bTheta)$ is a compressed model (hence a feasible weight vector), while \w\ is not (although it will be very close to $\bDelta(\bTheta)$). Hence, the solution (feasible and (near-)optimal) is $\bDelta(\bTheta)$.

\paragraph{Other comments}

In the derivation of the LC algorithm we used a quadratic penalty to penalize violations of the equality constraint $\w = \bDelta(\bTheta)$. This is convenient because it makes the L step easy with gradient-based optimization (it behaves like a form of weight decay on the loss), and the C step is also easy for some compression forms (quantization can be done by $k$-means, low-rank approximation by the SVD). However, non-quadratic penalty forms $P(\w,\bDelta(\bTheta))$ may be convenient in other situations.

Note that the C step can also be seen as trying to predict the weights \w\ from the low-dimensional parameters \bTheta\ via the mapping \bDelta. In this sense, compression is a machine learning problem of modeling ``data'' (the $P$ weights) using a low-dimensional space (the $Q$ parameters). This was noted by \citet{Denil_13a} in the context of low-rank models, but it applies generally in our framework. See also our discussion of parametric embeddings in section~\ref{s:MAC-PE}. In fact, model fitting itself in machine learning can be seen as compressing the training set into the model parameters.

\subsection{Direct compression (DC) and the beginning of the path}
\label{s:LC:DC}

In the LC algorithm, the parameters trace a path $(\w(\mu),\bTheta(\mu))$ for $\mu \ge 0$, and the solution is obtained for $\mu \rightarrow \infty$, when the constraints are satisfied and we achieve an optimal compressed model. The beginning of the path, for $\mu \rightarrow 0^+$, has a special meaning: it corresponds to the \emph{direct compression} (training the reference model and then compressing its weights), as we show next. (We write $\mu \rightarrow 0^+$ rather than $\mu = 0$ because the latter defines a problem without \bTheta.)

Taking $\mu \rightarrow 0^+$ in~\eqref{e:QP} or~\eqref{e:augLag} and assuming an initial $\blambda=\0$, we obtain
\begin{equation}
  \label{e:w0}
  \w(0^+) = \argmin_{\w}{ L(\w) } \equiv \overline{\w},
\end{equation}
since the $\mu$ term is negligible, and
\begin{equation}
  \label{e:theta0}
  \bTheta(0^+) = \bPi(\overline{\w}) = \argmin_{\bTheta}{ \norm{\overline{\w} - \bDelta(\bTheta)}^2 } \equiv \bTheta^{\text{DC}},
\end{equation}
i.e., the orthogonal projection of $\overline{\w}$ on the feasible set (up to local optima in both \w\ and \bTheta), recalling the discussion of section~\ref{s:compression-as-projection}. Hence, the path starts at $(\w(0^+),\bTheta(0^+)) = (\overline{\w},\bTheta^{\text{DC}})$, which corresponds to the \emph{direct compression}: training the large, reference model and then compressing its weights (note that in DC we discard $\overline{\w}$ and keep only $\w^{\text{DC}} \equiv \bDelta(\bTheta^{\text{DC}})$, i.e., the compressed model). This is not optimal in the sense of problem~\eqref{e:compression-problem} because the compression ignores the learning task; the best compression of the weights need not be the best compressed model for the task.

The constrained optimization view shows that, if an optimal uncompressed model $\overline{\w}$ is feasible, i.e., there is a $\bTheta \in \bbR^Q$ with $\bDelta(\bTheta) = \overline{\w}$, then it is optimal compressed, since the compression has zero error, and in this case $\overline{\w} = \w^{\text{DC}} = \w^*$ (and there is no need to optimize with the LC algorithm). But, generally, compression will increase the loss, the more so the larger the compression level (so the smaller the feasible set and the larger the distance $\norm{\smash{\overline{\w} - \bDelta(\bTheta^{\text{DC}})}}$ to the DC model). Therefore, we should expect that, with low compression levels, direct compression will be near-optimal, but as we compress more---which is our goal, and critical in actual deployment in mobile devices---it will become worse and worse in loss wrt the optimal compressed model $\w^*$. Hence, high compression rates require the better LC optimization. Plot 3 in figure~\ref{f:LC-illustration} illustrates this. Indeed, the suboptimality of direct compression compared to the result of the LC algorithm becomes practically evident in experiments compressing neural nets as we push the compression level \citep{CarreirIdelbay17a,CarreirIdelbay17b}. In section~\ref{s:DC}, we discuss existing work related to direct compression.

\section{Convergence results for the LC algorithm}
\label{s:conv}

The quadratic penalty and augmented Lagrangian methods belong to the family of homotopy (path-following) algorithms, where the minima $(\w(\mu),\bTheta(\mu))$ of $Q(\w,\bTheta;\mu)$ or $\calL_A(\w,\bTheta,\blambda;\mu)$ define a path for $\mu \ge 0$ and the solution we want is at $\mu \rightarrow \infty$. We give a theorem based on the QP; similar results are possible for the AL. Assume the loss $L(\w)$ and the decompression mapping $\bDelta(\bTheta)$ are continuously differentiable wrt their arguments, and that the loss is lower bounded.
\begin{thm}
  \label{th:LC}
  Consider the constrained problem~\eqref{e:compression-problem} and its quadratic-penalty function $Q(\w,\bTheta;\mu)$ of~\eqref{e:QP}. Given a positive increasing sequence $(\mu_k) \rightarrow \infty$, a nonnegative sequence $(\tau_k) \rightarrow 0$, and a starting point $(\w^0,\bTheta^0)$, suppose the QP method finds an approximate minimizer $(\w^k,\bTheta^k)$ of $Q(\w^k,\bTheta^k;\mu_k)$ that satisfies $\norm{\smash{\nabla_{\w,\bTheta} \, {Q(\w^k,\bTheta^k;\mu_k)}}} \le \tau_k$ for $k=1,2,\dots$ Then, $\lim_{k\rightarrow\infty}{ \big( (\w^k,\bTheta^k) \big) } = (\w^*,\bTheta^*)$, which is a KKT point for the problem~\eqref{e:compression-problem}, and its Lagrange multiplier vector has elements $\blambda^*_i = \smash{\lim_{k\rightarrow\infty}{ \big( -\mu_k \, (\w^k_i - \bDelta(\bTheta^k_i)) \big) }}$, $i=1,\dots,P$.
\end{thm}
\begin{proof}
  It follows by applying theorem~17.2 in \citet{NocedalWright06a}, quoted in appendix~\ref{s:QP-conv}, to the constrained problem~\eqref{e:compression-problem} and by noting that 1) $\lim_{k\rightarrow\infty}{ \big( (\w^k,\bTheta^k) \big) } = (\w^*,\bTheta^*)$ exists and 2) that the constraint gradients are linearly independent. We prove these two statements in turn. First, the limit of the sequence $((\w^k,\bTheta^k))$ exists because the loss and hence the QP function $Q(\w,\bTheta;\mu)$ are lower bounded and have continuous derivatives. Second, the constraint gradients are linearly independent at any point $(\w,\bTheta)$ and thus, in particular, at the limit $(\w^*,\bTheta^*)$. To see this, note that the Jacobian of the constraint function $\w - \bDelta(\bTheta) = \0$ wrt $(\w,\bTheta)$ is the $P \times (P+Q)$ matrix $\left( \I_P\ \ \nabla_{\bTheta}{ \bDelta(\bTheta) } \right)$, whose rank is obviously $P$, and so is full-rank.
\end{proof}
Stated otherwise, the LC algorithm defines a continuous path $(\w(\mu),\bTheta(\mu))$ which, under some mild assumptions (essentially, that we minimize $Q(\w,\bTheta;\mu)$ increasingly accurately as $\mu\rightarrow\infty$), converges to a stationary point (typically a minimizer) of the constrained problem~\eqref{e:compression-problem}. With convex problems, there is a unique path leading to the solution. With nonconvex problems, there are multiple paths, each leading to a local optimum. As with any nonconvex continuous optimization problem, convergence may occur in pathological cases to a stationary point of the constrained problem that is not a minimizer, but such cases should be rare in practice.

Computationally, it is better to approach the solution by following the path from small $\mu$, because $Q$ (or $\calL_A$) become progressively ill-conditioned as $\mu \rightarrow \infty$. While ideally we would follow the path closely, by increasing $\mu$ slowly from $0$ to $\infty$, in practice we follow the path loosely to reduce the runtime, typically using a multiplicative schedule for the penalty parameter, $\mu_k = \mu_0 a^k$ for $k = 0,1,2\dots$ where $\mu_0>0$ and $a>1$. If, after the first iteration, the iterates get stuck at the direct compression value, this is usually a sign that we increased $\mu$ too fast. The smaller $\mu_0$ and $a$, the more closely we follow the path.

This theorem applies to a number of common losses used in machine learning, and to various compression techniques, such as low-rank factorization, which are continuously differentiable. However, it does not apply to some popular compression forms, specifically quantization and pruning, which generally give rise to NP-complete problems. Indeed, consider one of the simplest models: least-squares linear regression, which defines a quadratic loss over the weights, whose solution is given by a linear system. Forcing the weights to be either $-1$ or $+1$ (quantization by binarization) defines a binary quadratic problem over the weights, which is NP-complete \citep{GareyJohnson79a}. Forcing a given proportion of the weights to be zero (pruning) is an $\ell_0$-constrained problem, also NP-complete \citep{Nataraj95a}. While we cannot expect the LC algorithm to find the global optimum of these problems, we can expect reasonably good results in the following sense. 1) The LC algorithm is still guaranteed to converge to a weight vector \w\ that satisfies the constraints (having elements in $\{-1,+1\}$ or having the given proportion of elements be zero, in those examples), hence it will always converge to a validly compressed model. 2) Because the L step minimizes (partially) the loss, the convergence will likely be to a low-loss point (even if not necessarily optimal).

\subsection{Choice of learning rate in the L step with large-scale optimization\protect\footnote{In this section, we use the subindex $t$ to indicate iterates, such as $\w_t$, \emph{within the L step}. These are different from the iterates \emph{of the LC algorithm}, denoted as $(\w^k,\bTheta^k)$ in theorem~\ref{th:LC}.}}
\label{s:SGD-clip}

Theorem~\ref{th:LC} states that for convergence to occur, the L and C steps must be solved increasingly accurately. This is generally not a problem for the C step, as it is usually solved by an existing compression algorithm. The L step needs some consideration. The objective function over \w\ in the L step has the form of the original loss $L$ plus a very simple term, a separable quadratic function:
\begin{equation}
  \label{e:Lstep-SGD}
  \min_{\w}{ Q(\w) = L(\w) + \frac{\mu}{2} \norm{\w - \w'}^2 }
\end{equation}
where $\w' = \bDelta(\bTheta)$ for the QP~\eqref{e:QP} and $\w' = \bDelta(\bTheta) - \smash{\frac{1}{\mu}} \blambda$ for the AL~\eqref{e:augLag}. Intuitively, optimizing~\eqref{e:Lstep-SGD} should not be very different from optimizing the loss (which we had to do in order to obtain the reference model). Indeed, gradient-based optimization is straightforward, since the gradient of~\eqref{e:Lstep-SGD} simply adds $\mu (\w-\w')$ to the gradient of the loss. Many optimization algorithms can be used to solve this depending on the form of $L$, such as a modified Newton method with line searches or even solving a linear system if $L$ is quadratic. However, for large-scale problems (large datasets and/or large dimension of \w) we are particularly interested in gradient-based optimization without line searches, such as gradient descent with a fixed step size, or stochastic gradient descent (SGD) with learning rates (step sizes) satisfying a Robbins-Monro schedule. Convergence without line searches requires certain conditions on the step size, and these must be corrected to account for the fact that the quadratic $\mu$-term increases as $\mu$ increases, since then the gradient term $\mu (\w-\w')$ also increases, and this can cause problems (such as large, undesirable jumps in early epochs of each new L step if using SGD). We consider two cases of optimizing the L step: using gradient descent with a fixed step size, and using SGD.

\subsubsection{Optimization of a convex loss using gradient descent with a fixed step size}

As is well known from convex optimization arguments (see proofs in appendix~\ref{s:L-step-learnrate:cvx}), if the loss $L(\w)$ is convex differentiable with Lipschitz continuous gradient with Lipschitz constant $M > 0$, then training the reference model can be done by gradient descent with a fixed step size $\smash{\frac{1}{M}}$. In the L step, the objective function~\eqref{e:Lstep-SGD} is strictly convex and therefore it has a unique minimizer. Gradient descent with a fixed step size $\smash{\frac{1}{M + \mu}}$, $\w_{t+1} = \w_{t} - \smash{\frac{1}{M + \mu}} \big( \nabla L(\w_{t}) + \mu (\w_{t} - \w') \big)$, converges to the minimizer linearly with rate $\smash{\frac{M}{M + \mu}} \in (0,1)$ from any initial point. Hence, we simply need to adjust the step size from $\smash{\frac{1}{M}}$ in the reference model to $\smash{\frac{1}{M + \mu}}$ in the L step. Although the step size becomes smaller as $\mu$ increases, the convergence becomes faster. The reason is that the objective function becomes closer to a separable quadratic function, whose optimization is easier.

\subsubsection{Optimization using stochastic gradient descent (SGD)}

With neural nets, which involve a nonconvex loss and a large dataset, SGD is the usual optimization procedure. The loss takes the form $L(\w) = \sum^N_{n=1}{ L_n(\w) }$ where $L_n$ is the loss for the $n$th training point and $N$ is large, so it is too costly to evaluate the gradient exactly. It is practically preferable to take approximate gradient steps based on evaluating the gradient at step $t$ on a minibatch $\calB_{t} \subset \{1,\dots,N\}$, hence each step can be seen as a gradient step using a noisy gradient:
\begin{equation}
  \label{e:SGD}
  \w_{t+1} = \w_{t} - \eta_{t} \nabla_{\calB_{t}} L(\w_{t}) \qquad \text{where} \qquad \nabla_{\calB_{t}} L(\w_{t}) = \sum_{n \in \calB_{t}}{ \nabla L_n(\w_{t}) } = \nabla L(\w_{t}) + \text{noise}_{t}.
\end{equation}
The convergence theorems for this stochastic setting are very different from those assuming exact gradients, most notably in the requirement that the step sizes (learning rates) $\eta_{t}$ must tend to zero, at a certain speed, as $t \rightarrow \infty$, which we call a \emph{Robbins-Monro schedule}:
\begin{equation}
  \label{e:RM-schedule}
  \text{Robbins-Monro schedule $\{\eta_{t}\}^{\infty}_{t=0}$:} \quad \eta_{t} > 0\ \forall t = 0,1,2\dots, \quad \sum^{\infty}_{t=0}{ \eta_{t} } = \infty, \quad \sum^{\infty}_{t=0}{ \eta^2_{t} } < \infty.
\end{equation}
Appendix~\ref{s:L-step-learnrate:SGD} gives detailed theorems with sufficient conditions for convergence (to a stationary point of the loss) in the case where the noise is deterministic (theorem~\ref{th:SGD-det}), in particular for the incremental gradient algorithm (theorem~\ref{th:SGD-incr-grad}), and when the noise is stochastic (theorem~\ref{th:SGD-stoch}). These conditions include Lipschitz continuity of the loss gradient, a condition on the noise to be small enough, and that the learning rates satisfy a Robbins-Monro schedule. The convergence rate is sublinear and much slower than with exact gradients (discussed in the previous section). In practice, the schedule typically has the form $\eta_{t} = \frac{\alpha}{\beta + t}$ where $\alpha,\beta > 0$ are determined by trial-and-error on a subset of the data. Unfortunately, the convergence theory for SGD is not very practical: apart from the conditions on the loss and the noise (which can be hard to verify but are mild), all the theory tells us is that using a Robbins-Monro schedule will lead to convergence. However, the performance of SGD is very sensitive to the schedule and selecting it well is of great importance. Indeed, SGD learning of neural nets is notoriously tricky and considerable trial and error in setting its hyperparameters (learning rate, momentum rate, minibatch size, etc.\@) is unavoidable in practice.

Consider now the $Q(\w)$ objective function~\eqref{e:Lstep-SGD} of the L step. The SGD updates take the form:
\begin{equation}
  \label{e:SGD-Lstep}
  \w_{t+1} = \w_{t} - \eta_{t} \left( \nabla_{\calB_{t}} L(\w_{t}) + \frac{\mu}{2} \norm{\w - \w'}^2 \right) \text{ where } \nabla_{\calB_{t}} L(\w_{t}) = \sum_{n \in \calB_{t}}{ \nabla L_n(\w_{t}) } = \nabla L(\w_{t}) + \text{noise}_{t}.
\end{equation}
Our concern is, given a good Robbins-Monro schedule $\{\eta_{t}\}^{\infty}_{t=0}$ for the reference model (i.e., for optimizing $L(\w)$ alone), should the schedule be modified for optimizing $Q$, and if so how? In theory, no change is needed, because the only condition that the convergence theorems require of the schedule is that it be Robbins-Monro. In practice, this requires some consideration. The addition of the $\mu$-term to the loss has two effects. First, since its exact gradient $\mu (\w-\w')$ is fast to compute, the noise in the gradient of $Q$ will (usually) be smaller, which is a good thing for convergence purposes. Second, and this is the practical problem, since $\mu$ increases towards infinity its gradient becomes progressively large%
\footnote{This makes the situation different from weight decay, which uses a similar objective function, of the form $L(\w) + \lambda \norm{\w}^2$, but where $\lambda$ is fixed and usually very small.}.
Hence, the \emph{early} weight updates in the L step (which use a larger learning rate) may considerably perturb \w, sending it away from the initial \w\ (provided by warm-start from the previous iteration's C step). While convergence to some minimizer (or stationary point in general) of $Q$ is still guaranteed with a Robbins-Monro schedule, this may be much slower and occur to a different local minimizer. This makes the overall optimization unstable and should be avoided.

We can solve these problems by using a schedule $\{\eta'_{t}\}^{\infty}_{t=0}$ that both satisfies the Robbins-Monro conditions and would lead to convergence of the $\mu$-term alone:
\begin{equation}
  \label{e:SGD-clip}
  \textcolor{blue}{\text{Clipped schedule $\{\eta'_{t}\}^{\infty}_{t=0}$:} \quad \eta'_{t} = \min{ \left( \eta_{t},\frac{1}{\mu} \right) },\ t = 0,1,2\dots}
\end{equation}
That is, the first iterations will use a learning rate $\smash{\frac{1}{\mu}}$, and then switch to the original $\eta_{t}$ schedule (when $\eta_{t} \le \smash{\frac{1}{\mu}}$). The justification is provided by the following two theorems.
\begin{thm}
  \label{th:quad-term-schedule}
  Consider minimizing a function $f(\w)$ using gradient descent with a fixed step size $\eta > 0$, i.e., we iterate $\w_{t+1} = \w_{t} - \eta \nabla f(\w_{t})$ for $t = 0,1,2\dots$ and some $\w_0$. If $f(\w) = \smash{\frac{\mu}{2} \norm{\w - \w'}^2}$, then that sequence converges linearly to the minimizer $\w'$ iff $\eta \in \big(0,\frac{2}{\mu}\big)$. Also, $\eta = \frac{1}{\mu}$ converges to the minimizer in one iteration.
\end{thm}
\begin{proof}
  A gradient descent step with step size $\eta > 0$ is $\w_{t+1} = \w_{t} - \eta \mu (\w_{t} - \w') = (1-\eta\mu) \w_{t} + \eta\mu\w'$. So $\eta = \smash{\frac{1}{\mu}}$ yields $\w_{t+1} = \w'$. And $\w_{t+1} - \w' = (1-\eta\mu) (\w_{t} - \w') = (1-\eta\mu)_{t+1} (\w_0 - \w')$, which converges linearly to zero if $\abs{1-\eta\mu} < 1 \Leftrightarrow \eta \in \big(0,\smash{\frac{2}{\mu}}\big)$.
\end{proof}
\begin{thm}
  \label{th:clipped-RM}
  Consider a learning rate schedule $\{\eta_{t}\}^{\infty}_{t=0}$ that satisfies the Robbins-Monro conditions~\eqref{e:RM-schedule} and let $\mu > 0$. Then the schedule $\{\eta'_{t}\}^{\infty}_{t=0}$ with $\eta'_{t} = \min{ \big( \eta_{t},\frac{1}{\mu} \big) },\ t = 0,1,2\dots$ also satisfies~\eqref{e:RM-schedule}.
\end{thm}
\begin{proof}
  We have by assumption that a) $\eta_{t} > 0$ $\forall t$, b) $\sum^{\infty}_{t=0}{ \eta_{t} } = \infty$ and c) $\sum^{\infty}_{t=0}{ \eta^2_{t} } < \infty$. Obviously, $\eta'_{t} > 0$ $\forall t$. From c) we have that $\lim_{t \rightarrow \infty}{\eta_{t}} = 0$, so there exists a $T \ge 0$ such that $\eta_{t} < \smash{\frac{1}{\mu}}$ $\forall t \ge T$, and $\sum^{\infty}_{t=T}{\eta_{t}} = \infty$. Hence $\sum^{\infty}_{t=T}{\eta'_{t}} = \sum^{\infty}_{t=T}{\eta_{t}} = \infty$, and $\sum^{\infty}_{t=0}{\eta'_{t}} = \infty$. Finally, since $0 < \eta'_{t} \le \eta_{t}$ $\forall t$, i.e., the sequence $(\eta'_{t})$ is majorized by the sequence $(\eta_{t})$, then $\sum^{\infty}_{t=0}{(\eta'_{t})^2} \le \sum^{\infty}_{t=0}{\eta^2_{t}} < \infty$.
\end{proof}
Theorem~\ref{th:quad-term-schedule} tells us we should use learning rates below $\smash{\frac{2}{\mu}}$ and suggests using $\smash{\frac{1}{\mu}}$ (particularly as $\mu$ increases, since then the $\mu$-term becomes dominant). Theorem~\ref{th:clipped-RM} guarantees that clipping a Robbins-Monro schedule remains Robbins-Monro. Hence, the clipped schedule makes sure that the initial, larger updates do not exceed $\smash{\frac{1}{\mu}}$ (the optimal step size for the $\mu$-term), and otherwise leaves it unchanged. This then ensures that the first steps do not unduly perturb the initial \w, while convergence to a minimum of $Q(\w)$ is guaranteed since the schedule is Robbins-Monro and $Q$ has Lipschitz continuous gradient if $L$ does (as long as the noise condition in the convergence theorems holds).

In a nutshell, our practical recommendation is as follows: first, we determine by trial and error a good schedule for the reference model (i.e., which drives the weight vector \w\ to close to a local minimizer $\overline{\w}$ of the loss $L$ as fast as possible). Then, we use the clipped schedule in the L step for $\mu > 0$. We have done this in experiments on various compression forms \citep{CarreirIdelbay17a,CarreirIdelbay17b} and found it effective.

\section{Relation of the LC algorithm with other algorithms}
\label{s:related-alg}

\subsection{One algorithm, many compression types}

We emphasize that the specific form of our LC algorithm follows necessarily from the definition of the compression technique in the constraints of problem~\eqref{e:compression-problem}. Some work on neural net compression is based on modifying the usual training procedure (typically backpropagation with SGD) by manipulating the weights on the fly in some ad-hoc way, such as binarizing or pruning them, but this usually has no guarantee that it will solve problem~\eqref{e:compression-problem}, or converge at all. In our framework, the LC algorithm (specifically, the C step) follows necessarily from the constraints that define the form of compression in~\eqref{e:compression-problem}. For example, for quantization \citep{CarreirIdelbay17a} and low-rank compression the C step results in $k$-means and the SVD, respectively, because the C step optimization ``$\smash{\min_{\bTheta}{ \norm{ \w - \bDelta(\bTheta) }^2 }}$'' takes the form of a quadratic distortion problem in both cases. For pruning using the $\ell_0$ norm \citep{CarreirIdelbay17b}, the optimization in the C step results in a form of weight magnitude thresholding. There is no need for any ad-hoc decisions in the algorithm.

\subsection{Direct compression and retraining approaches}
\label{s:DC}

Our formulation of what an optimal solution to the model compression problem is and the form of the LC algorithm allows us to put some earlier work into context.

\subsubsection{Direct compression approaches}

Direct compression (DC) consists of training the reference model and then compressing its weights. As shown in section~\ref{s:LC:DC}, DC corresponds to the beginning of the iterates' path in the LC algorithm, and is suboptimal, that is, it does not produce the compressed model with lowest loss. That said, direct compression is an appealing approach: it is an obvious thing to do, it is simple, and it is fast, because it does not require further training of a reference model (and hence no further access to the training set). Indeed, particular instances of DC corresponding to particular compression techniques have been recently applied to compress neural nets (although presumably much earlier attempts exist in the literature). These include quantizing the weights of the reference net with $k$-means \citep{Gong_15a}, pruning the weights of the reference net by zeroing small values \citep{Reed93a} or reducing the rank of the weight matrices of the reference net using the SVD \citep{Sainat_13a,Jaderb_14a,Denton_14a}. However, the LC algorithm is nearly as simple as direct compression: it can be seen as iterating the direct compression but with a crucial link between the L and C steps, the $\smash{\frac{\mu}{2} \norm{\w - \bDelta(\bTheta)}^2}$ term. Practically, this is not much slower, given that we have to train the reference model anyway. Since the C steps are usually negligible compared to the L steps, the LC algorithm behaves like training the reference model for a longer time.

\subsubsection{Retraining after direct compression}

As we mentioned in section~\ref{s:LC:DC}, the result of direct compression can have a large loss, particularly for large compression levels. A way to correct for this partially is to retrain the compressed model, and this has been a standard approach with neural net pruning \citep{Reed93a}. Here, we first train the reference net and then prune its weights in some way, e.g.\ by thresholding small-magnitude weights. This gives the direct compression model if using sparsity-inducing norms (see \citealp{CarreirIdelbay17b}). Finally, we optimize the loss $L$ again but only over the remaining, unpruned weights. This reduces the loss, often considerably. However, it loses the advantage of DC of not having to retrain the net (which requires access to the training set), and it is still suboptimal, since generally the set of weights that were pruned is not the set that would give the lowest loss. The LC algorithm consistently beats retraining for pruning, particularly for higher compression levels (see \citealp{CarreirIdelbay17b}).

\subsubsection{Iterated direct compression (iDC)}
\label{s:iDC}

Imagine we iterate the direct compression procedure. That is, we optimize $L(\w)$ to obtain $\overline{\w}$ and then compress it into $\bTheta^{\text{DC}}$. Next, we optimize $L(\w)$ again but initializing \w\ from $\bDelta(\bTheta^{\text{DC}})$, and then we compress it; etc. Our argument in section~\ref{s:LC:DC} implies that nothing would change after the first DC and we would simply cycle between $\overline{\w}$ and $\bDelta(\bTheta^{\text{DC}})$ forever. In fact, several DC iterations may be needed for this to happen, for two reasons. 1) With local optima of $L(\w)$, we might converge to a different optimum after the compression (see fig.~\ref{f:LC-illustration} plot 2). However, sooner rather than later this will end up cycling between a local optimum of $L(\w)$ and its compressed model. Still, this improves over the DC optimum. 2) A more likely reason in practice are inexact compression or learning steps. This implies the iterates never fully reach either $\overline{\w}$ or $\bDelta(\bTheta^{\text{DC}})$ or both, and keep oscillating forever in between. This is particularly the case if training neural nets with stochastic gradient descent (SGD), for which converging to high accuracy requires long training times.

We call the above procedure ``iterated direct compression (iDC)''. A version of this for quantization has been proposed recently (``trained quantization'', \citealp{Han_15a}), although without the context that our constrained optimization framework provides. In our experiments elsewhere \citep{CarreirIdelbay17a}, we verify that neither DC not iDC converge to a local optimum of problem~\eqref{e:compression-problem}, while our LC algorithm does.

\subsection{Other algorithms beyond model compression}
\label{s:MAC-PE}

\subsubsection{The method of auxiliary coordinates (MAC)}

Overall, we derive our LC algorithm by applying the following design pattern to solve the compression problem: 1) introducing auxiliary variables \w\ in eq.~\eqref{e:compression-problem}, 2) handling the constraints via penalty methods (QP or AL) in eqs.~\eqref{e:QP}--\eqref{e:augLag}, and 3) optimizing the penalized function using alternating optimization over the original variables \bTheta\ and the auxiliary variables \w\ in eqs.~\eqref{e:Lstep}--\eqref{e:Cstep}. This design pattern is similar to that used in the \emph{method of auxiliary coordinates (MAC)} for optimizing nested systems such as deep nets \citep{CarreirWang12a,CarreirWang14a}, i.e., involving functions of the form $\f(\x;\W) = \f_{K+1}(\dots \f_2(\f_1(\x;\W_1);\W_2)\dots;\W_{K+1})$, where \x\ is an input data point and \W\ are trainable parameters (e.g.\ weights in a deep net). Here, one introduces auxiliary coordinates per data point $\x_n$ of the form $\z_{k,n} = \f_k(\z_{k-1,n};\W_k)$ (where $\z_{0,n} \equiv \x_n$), for $k = 1,\dots,K$ and $n = 1,\dots,N$. Then, handling these constraints with a penalty method and applying alternating optimization yields the final algorithm. This alternates a ``maximization'' step that optimizes single-layer functions independently of each other (over the $\W_k$) with a ``coordination'' step that optimizes over the auxiliary coordinates (the $\z_{k,n}$) independently for each data point. Hence, in MAC the auxiliary coordinates are per data point and capture intermediate function values within the nested function, while in our LC algorithm the auxiliary variables duplicate the parameters of the model in order to separate learning from compression.

\subsubsection{Parametric embeddings}

MAC and the LC algorithm become identical in one interesting case: \emph{parametric embeddings}. A \emph{nonlinear embedding} algorithm seeks to map a collection of high-dimensional data points $\Y = (\y_1,\dots,\y_N)$ of $D \times N$ to a collection of low-dimensional projections $\X = (\x_1,\dots,\x_N)$ of $L \times N$ with $L < D$ such that distances or similarities between pairs of points $(\y_n,\y_m)$ are approximately preserved between the corresponding pairs of projections $(\x_n,\x_m)$. Examples of nonlinear embeddings are spectral methods such as multidimensional scaling \citep{BorgGroenen05a} or Laplacian eigenmaps \citep{BelkinNiyogi03b}, and truly nonlinear methods such as stochastic neighbor embedding (SNE) \citep{HintonRoweis03a}, $t$-SNE \citep{MaatenHinton08a} or the elastic embedding \citep{Carreir10a}. For example, the elastic embedding optimizes:
\begin{equation}
  \label{e:EE}
  E(\X) = \sum^N_{n,m=1}{a_{nm} \norm{\x_n-\x_m}^2} + \lambda \sum^N_{n,m=1}{\exp{(-\norm{\x_n-\x_m}^2)}}
\end{equation}
where $a_{nm}$ defines the similarity between $\y_n$ and $\y_m$ (the more positive $a_{nm}$ is, the more similar $\y_n$ and $\y_m$ are). Therefore, the first term attracts similar points, the second term repels all points, and the optimal embedding \X\ balances both forces (depending on the tradeoff parameter $\lambda > 0$). In a parametric embedding, we wish to learn a projection mapping $\F\mathpunct{:}\ \y \in \bbR^D \rightarrow \x \in \bbR^L$ rather than the projections $\y_1,\dots,\y_N$ themselves (so we can use \F\ to project a new point \y\ as $\F(\y)$). For the elastic embedding this means optimizing the following (where \F\ is a parametric mapping with parameters \bTheta, e.g.\ a linear mapping or a neural net):
\begin{equation}
  \label{e:parametric-EE}
  P(\bTheta) = \sum^N_{n,m=1}{a_{nm} \norm{\F(\y_n;\bTheta)-\F(\y_m;\bTheta)}^2} + \lambda \sum^N_{n,m=1}{\exp{\big(-\norm{\F(\y_n;\bTheta)-\F(\y_m;\bTheta)}^2\big)}}.
\end{equation}
To optimize this using MAC \citep{CarreirVladym15a}, we recognize the above as a nested mapping and introduce auxiliary coordinates $\z_n = \F(\y_n;\bTheta)$ for $n = 1,\dots,N$. The QP function is
\begin{equation}
  \label{e:PE:MAC-QP}
  Q(\Z,\bTheta;\mu) = E(\Z) + \frac{\mu}{2} \sum^N_{n=1}{ \norm{\z_n - \F(\y_n;\bTheta)}^2 } = E(\Z) + \frac{\mu}{2} \norm{\Z - \F(\Y;\bTheta)}^2
\end{equation}
and alternating optimization yields the following two steps:
\begin{itemize}
\item Over \Z, it has the form of a nonlinear embedding with a quadratic regularization:
  \begin{equation}
    \label{e:PE-Zstep}
    \min_{\Z}{ E(\Z) + \frac{\mu}{2} \norm{\Z - \F(\Y;\bTheta)}^2 }.
  \end{equation}
\item Over \bTheta, it has the form of a regression problem with inputs \Y\ and outputs \Z
  \begin{equation}
    \label{e:PE-Fstep}
    \min_{\bTheta}{ \sum^N_{n=1}{ \norm{\z_n - \F(\y_n;\bTheta)}^2 } }.
  \end{equation}
\end{itemize}
We can see this as an LC algorithm for model compression if we regard \Z\ as the uncompressed model (so $\overline{\Z} = \argmin_{\Z}{ E(\Z) }$ is the reference model) and \bTheta\ (or equivalently the projection mapping \F) as the compressed model. MAC and the LC algorithm coincide because in a parametric embedding each data point ($\y_n$) is associated with one parameter vector ($\z_n$). The decompression mapping is $\Z = \bDelta(\bTheta) = \F(\Y;\bTheta)$, which (approximately) recovers the uncompressed model by applying the projection mapping to the high-dimensional dataset. The compression step finds optimally the parameters \bTheta\ of \F\ via a regression fit. The learning step learns the regularized embedding \Z. ``Direct compression'' (called ``direct fit'' in \citealp{CarreirVladym15a}) fits \F\ directly to the reference embedding $\overline{\Z}$, which is suboptimal, and corresponds to the beginning of the path in the MAC or LC algorithm. Hence, in this view, \emph{parametric embeddings can be seen as compressed nonlinear embeddings}.

\section{Compression, generalization and model selection}
\label{s:gen}

In this paper we focus exclusively on compression as a mechanism to find a model having minimal loss and belonging to a set of compressed models, as precisely formulated in problem~\eqref{e:compression-problem}. However, generalization is an important aspect of compression, and we briefly discuss this.

Compression can also be seen as a way to prevent overfitting, since it aims at obtaining a smaller model with a similar loss to that of a well-trained reference model. This was noted early in the literature of neural nets, in particular pruning weights or neurons was seen as a way to explore different network architectures (see \citealp{Reed93a} and \citealp[ch.~9.5]{Bishop95a}). Soft weight-sharing \citep{NowlanHinton92a}, a form of soft quantization of the weights of a neural net, was proposed as a regularizer to make a network generalize better. More recently, weight binarization schemes have also been seen as regularizers \citep{Courbar_15a}.

Many recent papers, including our own work \citep{CarreirIdelbay17a,CarreirIdelbay17b}, report experimentally that the training and/or test error of compressed models is lower than that of the reference (as long as the compression level is not too large). Some papers interpret this as an improvement in the generalization ability of the compressed net. While this is to some extent true, there is a simpler reason for this (which we note in section~\ref{s:iDC}) that surely accounts for part of this error reduction: the reference model was not trained well enough, so that the continued training that happens while compressing reduces the error. This will generally be unavoidable in practice with large neural nets, because the difficulty in training them accurately will mean the reference model is close to being optimal, but never exactly optimal.

Model selection consists of determining the model type and size that achieves best generalization for a given task. It is a difficult problem in general, but much more so with neural nets because of the many factors that determine their architecture: number of layers, number of hidden units, type of activation function (linear, sigmoidal, rectified linear, maxpool\dots), type of connectivity (dense, convolutional\dots), etc. This results in an exponential number of possible architectures. \emph{Compression can be seen as a shortcut to model selection}, as follows. Instead of finding an approximately optimal architecture for the problem at hand by an expensive trial-and-error search, one can train a reference net that \emph{overestimates} the necessary size of the architecture (with some care to control overfitting). This gives a good estimate of the best performance that can be achieved in the problem. Then, one compresses this reference using a suitable compression scheme and a desired compression level, say pruning $p$\% of the weights or quantizing the weights using $\log_2{K}$ bits. Then, what our LC algorithm does is automatically search a subset of models of a given size (corresponding to the compression level). For example, the $\ell_0$-based pruning mechanism of \citet{CarreirIdelbay17b} uses a single $\kappa$ parameter (the total number of nonzero weights in the entire net) but implicitly considers all possible pruning levels for each layer of the net. This is much easier on the network designer than having to test multiple combinations of the number of hidden units in each layer. By running the LC algorithm at different compression levels $\kappa > 0$, one can determine the smallest model that achieves a target loss that is good enough for the application at hand. In summary, a good approximate strategy for model selection in neural nets is to train a large enough reference model and compress it as much as possible.

\section{Conclusion}
\label{s:concl}

We have described a general framework to obtain compressed models with optimal task performance based on casting compression, usually understood as a procedure, as constrained optimization in model parameter space. This accommodates many existing compression techniques and also sheds light on previous approaches that were derived procedurally and do not converge to an optimal compressed model, even if they are effective in practice. We have also given a general ``learning-compression'' (LC) algorithm, provably convergent under standard assumptions. The LC algorithm reuses two kinds of existing algorithms as a black-box, independently of each other: in the L step, \emph{a learning algorithm for the task loss} (such as SGD on the cross-entropy), which requires the training set, and whose form is independent of the compression technique; and in the C step \emph{a compression algorithm on the model parameters} (such as $k$-means or the SVD), whose form depends on the compression technique but is independent of the loss and training set. The L and C steps follow mathematically from the definition of the constrained problem; for example, the C step for quantization and low-rank compression results in $k$-means and the SVD, respectively, because the C step takes the form of a quadratic distortion problem in both cases. A model designer can try different losses or compression techniques by simply calling the appropriate routine in the L or C step.

In companion papers, we develop this framework for specific compression mechanisms and report experimental results that often exceed or are comparable with the published state of the art, but with the additional advantages of generality, simplicity and convergence guarantees. Because of this, we think our framework may be a useful addition to neural net toolboxes. Our framework also opens further research avenues that we are actively exploring.

\subsubsection*{Acknowledgements}

Work supported by NSF award IIS--1423515. I thank Yerlan Idelbayev (UC Merced) for useful discussions.

\appendix

\section{A convergence theorem for the quadratic-penalty method}
\label{s:QP-conv}

For reference, we quote a theorem from \citet{NocedalWright06a} that we use to prove convergence of the LC algorithm (using the QP) in our theorem~\ref{th:LC}.

Consider the equality-constrained problem
\begin{equation}
  \label{e:eq-constr}
  \min_{\x}{ f(\x) } \qquad \text{s.t.} \qquad \c(\x) = \0
\end{equation}
where $f\mathpunct{:}\ \bbR^n \rightarrow \bbR$ is continuously differentiable and $\c(\x) = (c_1(\x),\dots,c_m(\x))^T \in \bbR^m$ are $m$ equality constraints, also continuously differentiable. Define the quadratic-penalty function
\begin{equation}
  \label{e:eq-constr:QP}
  Q(\x;\mu) = f(\x) + \frac{\mu}{2} \sum^m_{i=1}{ c^2_i(\x) }
\end{equation}
where $\mu > 0$ is the penalty parameter. Assume we are given a sequence $0 < \mu_1 < \mu_2 < \dots < \infty$, a nonnegative sequence of tolerances $(\tau_k)$ with $(\tau_k) \rightarrow 0$ and an starting point $\x_0$. The quadratic-penalty method works by finding, at each iterate $k = 0,1,2\dots$, an approximate minimizer $\x_k$ of $Q(\x;\mu_k)$, starting at $\x_{k-1}$ and terminating when $\norm{ \nabla_{\x}{ Q(\x;\mu_k) } } \le \tau_k$ .

\begin{thm}
  Suppose that the tolerances and penalty parameters satisfy $(\tau_k) \rightarrow 0$ and $(\mu_k) \rightarrow \infty$. Then, if a limit point $\x^*$ of the sequence $(\x_k)$ is infeasible, it is a stationary point of the function $\smash{\norm{\c(\x)}}^2$. On the other hand, if a limit point $\x^*$ is feasible and the constraint gradients $\nabla c_i(\x^*)$, $i = 1,\dots,m$ are linearly independent, then $\x^*$ is a KKT point for the problem~\eqref{e:eq-constr}. For such points, we have for any infinite subsequence \calK\ such that $\lim_{k \in \calK}{\x_k} = \x^*$ that $\lim_{k \in \calK}{ \big( -\mu_k \, c_i(\x_k) \big) } = \lambda^*_i$, for $i = 1,\dots,m$, where $\blambda^*$ is the multiplier vector that satisfies the KKT conditions for problem~\eqref{e:eq-constr}.
\end{thm}
\begin{proof}
  See \citet[theorem~17.2]{NocedalWright06a}.
\end{proof}
Note that the QP method does not use the Lagrange multipliers in any way; the fact that $-\mu_k \, c_i(\x_k)$ tends to the Lagrange multiplier for constraint $i$ is a subproduct of the fact that the iterates converge to a KKT point. The AL method improves over the QP precisely by capitalizing on the availability of those estimates of the Lagrange multipliers.

\section{Learning rates for the L step: theorems and proofs}
\label{s:L-step-learnrate}

\subsection{Optimization of a convex loss using gradient descent with a fixed step size}
\label{s:L-step-learnrate:cvx}

First we present a few well-known results about gradient-based optimization for convex functions, with a short proof if possible, and then apply them to our L step objective function~\eqref{e:Lstep-SGD}. 

\subsubsection{Convergence theorems}

A function $f\mathpunct{:}\ \bbR^n \rightarrow \bbR$ is convex iff $\forall \x,\y \in \bbR^n$ and $\lambda \in [0,1]$: $ f(\lambda \x + (1 - \lambda) \y) \le \lambda f(\x) + (1 - \lambda) f(\y)$ (and strictly convex if the inequality is strict). Let $f$ be convex and differentiable. We say $f$ is strongly convex with constant $l > 0$ if $\forall \x,\y \in \bbR^n$: $f(\y) \ge f(\x) + \nabla f(\x)^T (\y-\x) + \smash{\frac{1}{2}} l \smash{\norm{\y-\x}^2}$. A function $\G\mathpunct{:}\ \bbR^n \rightarrow \bbR^m$ is Lipschitz continuous with Lipschitz constant $M > 0$ if $\forall \x,\y \in \bbR^n$: $\norm{\G(\x) - \G(\y)} \le M \norm{\x - \y}$. All norms are Euclidean in this section. Most of the statements apply if the convex function is defined on a convex subset of $\bbR^n$. For further details, see a standard reference such as \citet{Nester04a}.

\begin{lemma}
  \label{lem:Lip-sum}
  Let $f_1,f_2\mathpunct{:}\ \bbR^n \rightarrow \bbR$ be Lipschitz continuous with respective Lipschitz constants $M_1,M_2 > 0$. Then $f = f_1 + f_2$ is Lipschitz continuous with Lipschitz constant $M_1 + M_2$.
\end{lemma}
\begin{proof}
  $\forall \x,\y \in \bbR^n$: $\norm{f(\x) - f(\y)} = \norm{f_1(\x) - f_1(\y) + f_2(\x) - f_2(\y)} \le \norm{f_1(\x) - f_1(\y)} + \norm{f_2(\x) - f_2(\y)} \le (M_1+M_2) \norm{\x - \y}$, by applying the triangle inequality.
\end{proof}
\begin{lemma}
  \label{lem:cvx-diff}
  Let $f\mathpunct{:}\ \bbR^n \rightarrow \bbR$ be differentiable. Then $f$ is convex if and only if $\forall \x,\y \in \bbR^n$: $f(\y) \ge f(\x) + \nabla f(\x)^T (\y-\x)$.
\end{lemma}
\begin{proof}
  ($\Rightarrow$) Let $\x,\y \in \bbR^n$. Since $f$ is convex, we have $\forall \lambda \in [0,1]$: $(1 - \lambda) f(\x) + \lambda f(\y) \ge f( (1 - \lambda) \x + \lambda \y ) \Rightarrow f(\y) \ge f(\x) + \frac{1}{\lambda} \big( f(\x + \lambda (\y - \x)) - f(\x) \big)$, which tends to $f(\x) + \nabla f(\x)^T (\y-\x)$ as $\lambda \rightarrow 0$. \\
  ($\Leftarrow$) Let $\x,\y \in \bbR^n$, $\lambda \in [0,1]$ and $\z = \lambda \x + (1 - \lambda) \y$. Then, applying the assumption to $(\z,\x)$ and to $(\z,\y)$, we get the following two inequalities: $f(\x) \ge f(\z) + \nabla f(\z)^T (\x-\z)$ and $f(\y) \ge f(\z) + \nabla f(\z)^T (\y-\z)$. Multiplying the first by $\lambda$, the second by $1 - \lambda$ and summing, we obtain $\lambda f(\x) + (1 - \lambda) f(\y) \ge f(\z) = f(\lambda \x + (1 - \lambda) \y)$.
\end{proof}
\begin{lemma}
  \label{lem:cvx+Lip}
  Let $f\mathpunct{:}\ \bbR^n \rightarrow \bbR$ be convex and continuously differentiable, and $\nabla f\mathpunct{:}\ \bbR^n \rightarrow \bbR^n$ be Lipschitz continuous with Lipschitz constant $M > 0$. Then $\forall \x,\y \in \bbR^n$: $f(\y) \le f(\x) + \nabla f(\x)^T (\y-\x) + \frac{1}{2} M \norm{\y-\x}^2$.
\end{lemma}
\begin{proof}
  See \citet[lemma~1.2.3]{Nester04a}.
\end{proof}
\begin{lemma}
  \label{lem:cvx+quad}
  Let $f\mathpunct{:}\ \bbR^n \rightarrow \bbR$ be convex. Then $F(\x) = f(\x) + \frac{\mu}{2} \norm{\x - \aa}^2$ where $\aa \in \bbR^n$ and $\mu > 0$ is strongly convex with constant $\mu$.
\end{lemma}
\begin{proof}
  Call $g(\x) = \frac{\mu}{2} \norm{\x - \aa}^2$ so $F = f + g$. Then, since $f$ is convex: $f(\y) \ge f(\x) + \nabla f(\x)^T (\y-\x)$ $\forall \x,\y \in \bbR^n$, and one can verify by substitution that $g(\y) = g(\x) + \nabla \smash{g(\x)^T (\y-\x)} + \smash{\frac{1}{2} \mu \norm{\y-\x}^2}$ $\forall \x,\y \in \bbR^n$. Summing both expressions we get $F(\y) \ge F(\x) + \nabla F(\x)^T (\y-\x) + \frac{1}{2} \mu \norm{\y-\x}^2$ $\forall \x,\y \in \bbR^n$.
\end{proof}
\begin{thm}
  \label{th:cvx-graddesc}
  Let $f\mathpunct{:}\ \bbR^n \rightarrow \bbR$ be strongly convex with constant $m > 0$ and continuously differentiable, and let its gradient $\nabla f$ be Lipschitz continuous with Lipschitz constant $M > 0$. Given any $\x^0 \in \bbR^n$, define the sequence $\x_{t+1} = \x_{t} - \frac{1}{M} \nabla f(\x_{t})$ for $t = 0,1,2\dots$ Then $f$ has a unique global minimizer $\x^* \in \bbR^n$ and $f(\x_{t}) - f(\x^*) \le \smash{\left( 1 - \frac{m}{M} \right)_{t}} (f(\x^0) - f(\x^*))$ for $t = 0,1,2\dots$
\end{thm}
\begin{proof}
  Since $f$ is strongly convex it has a unique minimizer $\x^* \in \bbR^n$. From lemma~\ref{lem:cvx+Lip}:
  \begin{equation*}
    \forall \x,\y \in \bbR^n\mathpunct{:}\ f(\y) \le f(\x) + \nabla f(\x)^T (\y-\x) + \frac{1}{2} M \norm{\y-\x}^2.
  \end{equation*}
  By differentiating wrt \x, we see the RHS is minimal at $\y - \x = - \frac{1}{M} \nabla f(\x)$ and $f(\y) \le f(\x) - \frac{1}{2M} \norm{\nabla f(\x)}^2$. Applying this to $\x = \x_{t}$ and $\y = \x_{t+1}$, we get $\x_{t+1} = \x_{t} - \frac{1}{M} \nabla f(\x_{t})$ (a gradient descent step at $\x_{t}$ with step size $\frac{1}{M}$) and
  \begin{equation}
    \label{e:grad-decr}
    f(\x_{t+1}) \le f(\x_{t}) - \frac{1}{2M} \norm{\nabla f(\x_{t})}^2,
  \end{equation}
  which gives a lower bound on how much $f$ decreases from $\x_{t}$ to $\x_{t+1}$. Since $f$ is strongly convex:
  \begin{equation*}
    \forall \x,\y \in \bbR^n\mathpunct{:}\ f(\y) \ge f(\x) + \nabla f(\x)^T (\y-\x) + \frac{1}{2} m \norm{\y-\x}^2.
  \end{equation*}
  As a function of \y, the RHS is minimal at $\y = \x - \frac{1}{m} \nabla f(\x)$ and equals $f(\x) - \frac{1}{2m} \norm{\nabla f(\x)}^2$. Hence, $f(\y) \ge f(\x) - \smash{\frac{1}{2m} \norm{\nabla f(\x)}^2}$ $\forall \x,\y \in \bbR^n$. In particular, taking $\y = \x^*$ and $\x = \x_{t}$:
  \begin{equation*}
    f(\x^*) \ge f(\x_{t}) - \frac{1}{2m} \norm{\nabla f(\x_{t})}^2,
  \end{equation*}
  which gives an upper bound on how far $f(\x_{t})$ is from the minimum. Combining this with eq.~\eqref{e:grad-decr} we get $f(\x_{t+1}) - f(\x^*) \le \left( 1 - \frac{m}{M} \right) (f(\x_{t}) - f(\x^*))$ and the result follows.
\end{proof}
\begin{rmk}
  Theorem~\ref{th:cvx-graddesc} shows that, if $f$ is strongly convex with constant $m$ and its gradient is Lipschitz continuous with constant $M$, gradient descent with a constant step size $\frac{1}{M}$ converges linearly with rate $0 \le 1 - \frac{m}{M} < 1$ from any initial point.
\end{rmk}

\subsubsection{Application to the L step of the LC algorithm}

We now apply these results to our case. To train the reference model, we minimize the loss $L(\w)$. If we assume $L$ is convex differentiable with Lipschitz continuous gradient with Lipschitz constant $M > 0$, we could use gradient descent with a fixed step size $\smash{\frac{1}{M}}$ and have guaranteed convergence. In the L step, we minimize an objective function $Q(\w) = L(\w) + \smash{\frac{\mu}{2} \norm{\w - \w'}^2}$ with $\mu > 0$. We have that $Q$ is strongly convex with constant $\mu > 0$ (lemma~\ref{lem:cvx+quad}), so it has a unique minimizer, and its gradient is Lipschitz continuous with Lipschitz constant $M + \mu$ (lemma~\ref{lem:Lip-sum}). Hence, gradient descent on $Q$ with fixed step size $\smash{\frac{1}{M + \mu}}$, $\w_{t+1} = \w_{t} - \smash{\frac{1}{M + \mu}} \nabla Q(\w_{t})$, converges to the minimizer linearly with rate $1 - \smash{\frac{\mu}{M + \mu}} = \smash{\frac{M}{M + \mu}} \in (0,1)$ from any initial point. (If the loss is strongly convex with constant $0 < m < M$, then we have a faster rate of $\smash{\frac{M - m}{M + \mu}}$.)

\subsection{Optimization using stochastic gradient descent (SGD)}
\label{s:L-step-learnrate:SGD}

First we present some theorems about the convergence of SGD-type algorithms and then apply them to our L step objective function~\eqref{e:Lstep-SGD}. 

\subsubsection{Convergence theorems}

There is a large literature on convergence of SGD-type algorithms \citep{Pflug96a,Benven_90a,KushnerYin03a}, although the basic conditions for convergence are similar (Lipschitz continuity of the gradient, approximate descent search directions, bounded error if deterministic or unbiased error with bounded variance if stochastic, and Robbins-Monro step sizes). We quote theorems from \citet{BertsekTsitsik00a}, which are simple and consider the cases of deterministic errors, incremental gradient algorithm and stochastic errors.

Let $f\mathpunct{:}\ \bbR^n \rightarrow \bbR$ be a continuously differentiable function with Lipschitz continuous gradient $\nabla f$. Consider minimization of $f(\x)$ using the approximate descent method $\x_{t+1} = \x_{t} + \eta_{t} (\p_{t} + \e_{t})$ for $t = 0,1,2\dots$, where $\p_{t}$ is a search direction and $\e_{t}$ an error vector. We say that the step sizes are Robbins-Monro if they are positive and satisfy
\begin{equation}
  \label{e:SGD:RM}
  \sum^{\infty}_{t=0}{ \eta_{t} } = \infty, \qquad \sum^{\infty}_{t=0}{ \eta^2_{t} } < \infty.
\end{equation}

\begin{thm}[deterministic errors]
  \label{th:SGD-det}
  Let $\x_{t}$ be a sequence generated by the method $\x_{t+1} = \x_{t} + \eta_{t} (\p_{t} + \e_{t})$, where $\eta_{t}$ is a positive step size, $\p_{t}$ is a descent direction and $\e_{t}$ is a error vector. We assume the following:
  \begin{enumerate}
  \item There exist positive scalars $c_1$ and $c_2$ such that, for all $t$:
    \begin{equation}
      \label{e:SGD-det:dir}
      c_1 \norm{\nabla f(\x_{t})}^2 \le - \nabla f(\x_{t})^T \p_{t} \qquad \norm{\p_{t}} \le c_2 (1 + \norm{\nabla f(\x_{t})}).
    \end{equation}
  \item There exist positive scalars $p$ and $q$ such that, for all $t$:
    \begin{equation}
      \label{e:SGD-det:err}
      \norm{\e_{t}} \le \eta_{t} (q + p \norm{\nabla f(\x_{t})}).
    \end{equation}
  \item The step sizes are Robbins-Monro.
  \end{enumerate}
  Then either $f(\x_{t}) \rightarrow -\infty$ or else $f(\x_{t})$ converges to a finite value and $\lim_{t \rightarrow \infty}{ \nabla f(\x_{t}) } = \0$. Furthermore, every limit point of $\x_{t}$ is a stationary point of $f$.
\end{thm}
\begin{proof}
  See proposition 1 in \citet{BertsekTsitsik00a}.
\end{proof}

This theorem can be particularized to the following important case: 1) $f$ has the form $f(\x) = \smash{\sum^m_{i=1}{f_i(\x)}}$ and each $f_i$ is continuously differentiable with Lipschitz continuous gradient. 2) We use the incremental gradient algorithm, where we cycle over $f_1,\dots,f_m$, each time updating $\x_{t+1}$ using the gradient of $f_i$ at $\x_{t}$. This corresponds to SGD using a minibatch of size 1 and without reshuffling the dataset at each epoch.

\begin{thm}[incremental gradient method]
  \label{th:SGD-incr-grad}
  Let $\x_{t}$ be a sequence generated by the incremental gradient method. Assume that for some positive constants $C$ and $D$, and all $i = 1,\dots,m$, we have
  \begin{equation}
    \label{e:SGD-incr-grad:err}
    \norm{\nabla f_i(\x)} \le C + D \norm{\nabla f(\x)} \qquad \forall \x \in \bbR^n.
  \end{equation}
  Assume that the step sizes are Robbins-Monro. Then either $f(\x_{t}) \rightarrow -\infty$ or else $f(\x_{t})$ converges to a finite value and $\lim_{t \rightarrow \infty}{ \nabla f(\x_{t}) } = \0$. Furthermore, every limit point of $\x_{t}$ is a stationary point of $f$.
\end{thm}
\begin{proof}
  See proposition 2 in \citet{BertsekTsitsik00a}.
\end{proof}

Now we let the noise term $\e_{t}$ be stochastic. The $\sigma$-field $\calF_{t}$ should be interpreted as the history of the algorithm up to time $t$, just before $\e_{t}$ is generated, so that conditioning on $\calF_{t}$ represents conditioning on $\x_0,\p_0,\e_0,\dots,\x_{t-1},\p_{t-1},\e_{t-1},\x_{t},\p_{t}$.

\begin{thm}[stochastic errors]
  \label{th:SGD-stoch}
  Let $\x_{t}$ be a sequence generated by the method $\x_{t+1} = \x_{t} + \eta_{t} (\p_{t} + \e_{t})$, where $\eta_{t}$ is a deterministic positive step size, $\p_{t}$ is a descent direction and $\e_{t}$ is a random noise term. Let $\calF_{t}$ be an increasing sequence of $\sigma$-fields. We assume the following:
  \begin{enumerate}
    \setcounter{enumi}{-1}
  \item $\x_{t}$ and $\p_{t}$ are $\calF_{t}$-measurable.
  \item There exist positive scalars $c_1$ and $c_2$ such that, for all $t$:
    \begin{equation}
      \label{e:SGD-stoch:dir}
      c_1 \norm{\nabla f(\x_{t})}^2 \le - \nabla f(\x_{t})^T \p_{t} \qquad \norm{\p_{t}} \le c_2 (1 + \norm{\nabla f(\x_{t})}).
    \end{equation}
  \item We have, for all $t$ and with probability 1,
    \begin{equation}
      \label{e:SGD-stoch:mean-var}
      \mean{\e_{t}|\calF_{t}} = \0, \qquad \mean{\norm{\e^2_{t}}|\calF_{t}} \le A (1 + \norm{\nabla f(\x_{t})}^2),
    \end{equation}
    where $A$ is a positive deterministic constant.
  \item The step sizes are Robbins-Monro.
  \end{enumerate}
  Then, the following holds with probability 1: either $f(\x_{t}) \rightarrow -\infty$ or else $f(\x_{t})$ converges to a finite value and $\lim_{t \rightarrow \infty}{ \nabla f(\x_{t}) } = \0$. Furthermore, every limit point of $\x_{t}$ is a stationary point of $f$.
\end{thm}
\begin{proof}
  See proposition 3 in \citet{BertsekTsitsik00a}.
\end{proof}

\subsubsection{Application to the L step of the LC algorithm}

The theorems give sufficient conditions for convergence to a stationary point (usually a minimizer). Their effect in the L step if using SGD are as follows:
\begin{enumerate}
\item The conditions~\eqref{e:SGD-det:dir} and~\eqref{e:SGD-stoch:dir} on the search direction $\p_{t}$ are always satisfied since it equals the gradient ($\p_{t} = -\nabla f(\x_{t})$ in the theorems).
\item The conditions on the error or noise~\eqref{e:SGD-det:err}, \eqref{e:SGD-incr-grad:err} and~\eqref{e:SGD-stoch:mean-var} are hard to verify in general but should hold in many practical cases. \\
  It is tempting to think that the condition should hold for the L step objective function $Q(\w) = L(\w) + \smash{\frac{\mu}{2} \norm{\w - \w'}^2}$ if it holds for the loss $L(\w)$, since the gradient for the $\mu$-term has no error, but this is not true generally. To see this, assume the error condition holds for a function $f$ and consider a function $f + g$, where $\nabla g$ has no error, so the error comes from $\nabla f$ only ($f$ and $g$ correspond to the loss and $\mu$-term, respectively). We could pick $g$ adversarially so that the error is small wrt $\nabla f$ but big wrt $\nabla (f+g)$ (e.g.\ if $g = -f$). Although this could be solved by placing some assumptions on $g$, we may just as well assume the error condition on $f+g$.
\item The condition on the step sizes is that they be Robbins-Monro.
\end{enumerate}
Consequently, the theorems hold for SGD minimization of $Q(\w) = L(\w) + \smash{\frac{\mu}{2} \norm{\w - \w'}^2}$ if the error on $\nabla Q$ satisfies~\eqref{e:SGD-det:err}, \eqref{e:SGD-incr-grad:err} or~\eqref{e:SGD-stoch:mean-var} and the step sizes are Robbins-Monro. Hence, there is neither a simplification nor a complication of minimizing $Q$ over minimizing $L$ with SGD: the convergence theory leaves the choice of step sizes up to the user as long as they are Robbins-Monro.


\end{document}